\documentclass[twoside,11pt]{article}

\usepackage{jmlr2e}
\expandafter\let\csname equation*\endcsname\relax
\expandafter\let\csname endequation*\endcsname\relax
\usepackage{bm}
\usepackage{enumerate}
\usepackage{mathrsfs,color}
\usepackage{amsfonts}\usepackage{multirow}
\usepackage{amssymb}
\usepackage{dsfont}
\usepackage{blkarray}
\usepackage{amsmath}
\usepackage{float}
\usepackage{amssymb}
\usepackage[toc,page]{appendix}
\usepackage{graphicx,afterpage}
\usepackage{epstopdf}
\usepackage{cancel}
\usepackage{mathdots}
\usepackage{natbib}

\usepackage{mhequ}
\numberwithin{equation}{section}
\numberwithin{figure}{section}

\newcommand\tabcaption{\def\@captype{table}\caption}

\newtheorem{thm}{Theorem}[section]

\newtheorem{lem}[thm]{Lemma}
\newtheorem{prop}[thm]{Proposition}
\newtheorem{defn}[thm]{Definition}

\newtheorem{aspt}[thm]{Assumption}
\newtheorem{rem}[thm]{Remark}

\newcommand{\tr}{\textbf{tr}}

\newcommand{\E}{\mathbb{E}}
\newcommand{\Prob}{\mathbb{P}}

\newcommand{\reals}{\mathbb{R}}
\newcommand{\norm}[1]{\| #1 \|}
\newcommand{\unit}{\mathds{1}}

\newcommand{\blue}{}
\newcommand{\red}{}

\newcommand{\region}{\mathcal{R}}
\newcommand{\mat}[1]{#1}

\newcommand{\poly}{\text{poly}}
\newcommand{\abs}[1]{\left|#1\right|}

\newcommand{\calA}{\mathcal{A}}

\firstpageno{1}
\usepackage{lastpage}

\ShortHeadings{Stationary-Point Hitting Time and Ergodicity of SGLD}{Chen, Du, and Tong}

\begin{document}
	
	\title{On Stationary-Point Hitting Time and Ergodicity of Stochastic Gradient Langevin Dynamics}
	
	\author{\name Xi Chen \email xchen3@stern.nyu.edu\\
		\addr Stern School of Business\\
		New York University, New York, NY 10012, USA
		\AND
		\name Simon S. Du \email ssdu@ias.edu\\
		\addr School of Mathematics\\
		Institute for Advanced Study, Princeton, NJ 08540, USA	
		\AND
		\name Xin T. Tong \email mattxin@nus.edu.sg \\
		\addr Department of Mathematics\\
		National University of Singapore, Singapore 119076 , Singapore}
	
	\editor{}
	
	\maketitle
	
	\begin{abstract}
		\label{sec:abs}
Stochastic gradient Langevin dynamics (SGLD) is a fundamental algorithm in stochastic optimization.
Recent work by \citet{zhang2017hitting} presents an analysis for the hitting time of SGLD for the first and second order stationary points. The proof in  \citet{zhang2017hitting} is a two-stage procedure through bounding the Cheeger's constant, which is rather complicated and leads to loose bounds.  
In this paper, 
using intuitions from stochastic differential equations, 
we provide a \emph{direct} analysis for the hitting times of SGLD to the first and second order stationary points.  
Our analysis is straightforward. It only relies on basic linear algebra and probability theory tools.
Our direct analysis also leads to tighter bounds comparing to \cite{zhang2017hitting} and shows the explicit dependence of the hitting time on different factors, including dimensionality, smoothness, noise strength, and step size effects.  Under suitable conditions,  we show that the hitting time of SGLD to first-order stationary points can be dimension-independent. Moreover, we apply our analysis to study several important online estimation problems in machine learning, including linear regression, matrix factorization, and online PCA.

	\end{abstract}

	
	\section{Introduction}
	\label{sec:intro}

	Adding noise to the stochastic gradient descent algorithm has been found helpful in training deep neural networks~\citep{neelakantan2015adding} and in turn improving performance in many applications~\citep{kaiser2015neural,kurach2015neural,neelakantan2015neural,zeyer2017comprehensive}. 
For example, \citet{kurach2015neural} proposed a model named Neural Random-Access Machine to learn basic algorithmic operations like permutation, merge, etc.
In their experiments, they found adding noise can significantly improve the success rate.
However, theoretical understanding of gradient noise is still limited. 
In this paper, we study a particular noisy gradient-based algorithm, Stochastic Gradient Langevin Dynamics (SGLD). 
This algorithm intends to minimize a nonnegative objective function of form
\begin{equation}\label{eq:pop}
F(X)=\E_{\omega\sim \pi} f(X,\omega),\quad X\in \reals^d.
\end{equation}

Often, the distribution $\pi$  describes either a population distribution or an empirical distribution over a given dataset. At each iteration, the SGLD is updated by
\begin{align} 
X_{n+1}=X_n-\eta_{n+1} \nabla f(X_n, \omega_{n+1})+\delta_0\sqrt{\eta_{n+1}} \zeta_{n+1}. \label{eqn:sgld}
\end{align}
Here $\nabla f(X_n, \omega_{n+1})$ is the stochastic gradient of the objective function, $\omega_n$ are i.i.d. samples from $\pi$, $\zeta_{n+1}\sim \mathcal{N}(0, \mathbf{I}_d)$  is a standard $d$-dimensional Gaussian random vector, and $\eta_{n+1}$ is the step size parameter. 
As compared to the stochastic gradient descent (SGD), the SGLD imposes a larger step size for the noise term (i.e., $\sqrt{\eta_{n+1}}$ instead of $\eta_{n+1}$ for SGD \red{and $\eta_{n+1}$ will decrease to zero as $n\to \infty$}), which allows the SGLD to aptly navigate a landscape containing multiple critical points.
SGLD obtains its name because it is a discrete approximation to the continuous Langevin diffusion process, which can be described by the following stochastic differential equation (SDE)\begin{align}
d X_t = -\nabla F(X_t) dt + \delta_0 dW_t \label{eqn:langevin_diffusion},
\end{align}
\red{On the other hand, it is worthwhile noting that SGLD is different from the direct discretization of \eqref{eqn:langevin_diffusion}, which can be written as
\begin{equation}
\label{eqn:ULA}
X_{n+1}=X_n-\eta_{n+1} \nabla F(X_n)+\delta_0\sqrt{\eta_{n+1}} \zeta_{n+1}.
\end{equation} 
Algorithm \eqref{eqn:ULA} is known as the Unadjusted Langevin Algorithm (ULA). It is used in machine learning for sampling tasks, with some of its theoretical properties studied by \citet{dalalyan2017theoretical} (often assuming $F$ is convex). 
But the implementation of ULA requires the evaluation of the population gradient $\nabla F(X_n)$, which is often not available for machine learning application as the distribution $\pi$ in \eqref{eq:pop} is unavailable. 
}


Theoretically, SGLD has been studied from various perspectives.
Statistically, SGLD has been shown to have better generalization ability than the simple stochastic gradient descent  (SGD) algorithm~\citep{mou2017generalization,tzen2018local}.
From the optimization point of view, it is well known that SGLD traverses all stationary points asymptotically.
More recently, quantitative characterizations of the mixing-time are derived~\citep{raginsky2017non,xu2017global}.
However, bounds in these papers often depend on a quantity called the spectral gap of Langevin diffusion process (Equation~\eqref{eqn:langevin_diffusion}), 
	which in general has an exponential dependence on the dimension. 
We refer readers to Section~\ref{sec:rel} for more discussions.

While these bounds are pessimistic,  in many machine learning applications, 
finding a local minimum has already been useful. In other words, we only need the critical point hitting time  bound instead of mixing time bound.\footnote{See Section~\ref{sec:pre} for the precise definition.}
To the best of our knowledge, \citet{zhang2017hitting} is the first work studying the hitting time property of SGLD.
The analysis of \cite{zhang2017hitting} consists of two parts.
First, they defined a geometric quantity called Cheeger's constant of the target regions, and showed that
the Cheeger's constant of certain regions (e.g. the union of all approximate local minima) can be estimated.
Next, they derived a generic bound that relates the hitting time of SGLD and this Cheeger's constant.
	Through these two steps, they showed the hitting time of critical points only has a polynomial dependence on the dimension. 

However, due to this two-step analysis framework, the hitting time bound derived is often not tight. 
Technically, it is very challenging to accurately estimate the Cheeger's constant of the region of interest.
In particular, in many machine learning problems, useful structures such as low-rank and sparsity are available, which can be potentially exploited by SGLD to achieve faster convergence.
	Therefore, a natural research question is that: instead of using a two-stage approach, is there a direct method to obtain tighter hitting time bounds of SGLD that can incorporate underlying structural assumptions?

In this paper, we first consider the hitting time of SGLD to  first order and second order approximation stationary points.
For both types of stationary points, we provide a simple analysis of the hitting time of SGLD, which is motivated from the succinct continuous-time analysis.
Notably, our analysis only relies on basic real analysis, linear algebra, and probability tools.
In contrast to the indirect approach adopted by \citet{zhang2017hitting}, we directly estimate the hitting time of SGLD and thus obtain tighter bounds in terms of the dimension, error metric, and other problem-dependent quantities such as smoothness.
Comparing our results with  \cite{zhang2017hitting}, we have two main advantages.
First, our results are applicable to \emph{decreasing step sizes}---the setting widely used in practice. 
While previous analysis (including  \citet{zhang2017hitting})  mainly considers the constant step size setting, which limits its potential applications.
Second, in certain scenarios (see Section \ref{sec:examples}), we can obtain \emph{dimension independent} hitting time bounds, whereas bounds in \cite{zhang2017hitting} all require at least a polynomial dependence of dimension.

\red{In addition to hitting time, we further establish the ergodicity of SGLD, which is a unique property of SGLD and does not hold for classical stochastic gradient nor perturbed gradient methods (see, e.g., \cite{jin2017escape} and references therein). Roughly speaking, We show that SGLD can reach any given point in the state space given enough iterations. Please see Section \ref{sec:ergodic} for more details.
}

\subsection{Preliminaries and Problem Setup}
\label{sec:pre}
We use $\norm{\cdot}$ to denote the Euclidean norm of a finite-dimensional vector. We also use
$\langle u, v\rangle=u^Tv$ to denote the inner product of two vectors.
For a real symmetric matrix $\mat{A}$, we use $\lambda_{\max}\left(\mat{A}\right)$ to denote its largest eigenvalue and $\lambda_{\min}\left(\mat{A}\right)$ its smallest eigenvalue.
Let $O(\cdot)$ denote standard Big-O  notation, only hiding absolute constants.


In this paper we use either symbol $B$, $C$ or $D$ to denote problem dependent parameters.
The difference is, 
the $C$-constants can often be picked independently of the dimension, but the $B$ and $D$-constants usually increase with the dimension. 
Typical example can be, the spectral norm of the $d$-dimensional identity matrix remains $1$ for any $d$, but its trace increases linearly with $d$. 
On the other hand, the $B$ constants are in practice controlled by the batch sizes.
By writing two types of constants differently, it helps us to interpret the performance of SGLD in high dimensional settings. 
On the other hand, our results hold even if the $C$-constants increase with $d$, which is possible in certain scenarios. 
We denote $\eta_{1:n} = \sum_{i=1}^{n}\eta_i$, 
$\eta_{o:n} = \sum_{i=o}^{n}\eta_{i}$
	, $\eta_{1:n}^2 = \sum_{i=1}^{n}\eta_i^2$.
Throughout the paper, we use $0 < c < 1$ to denote an absolute constant, which may change from line to line.

In this paper we focus on SGLD defined in Equation~\eqref{eqn:sgld}.
Note that the stochastic gradient $\nabla f(X_n, \omega_{n+1})$ can be decomposed into two parts: one part is its mean $\nabla F(X_n)$, the other part is the difference between the stochastic gradient and its mean:
\begin{equation}
\label{sys:GDanom}
\xi_{n+1}:=\nabla f(X_n, \omega_{n+1})-\nabla F(X_n).
\end{equation}
We note that $\xi_{n+1}$  is a martingale difference series. 
With these notations, we can write the SGLD iterates as 
\begin{equation}
\label{sys:SGLD}
X_{n+1}=X_n-\eta_{n+1} (\nabla F(X_n)+\xi_{n+1})+\delta_0 \sqrt{\eta_{n+1}}\zeta_{n+1}
\end{equation}
Note that the step size $\eta_n$ is not a constant. This is crucial, since we know for SGD to converge, the step size needs to converge to zero. In practice, it is often picked as an inverse of polynomial, 
\begin{equation}
\label{sys:eta}
\eta_n=\eta_0 n^{-\alpha}, \quad \alpha \in [0,1),  
\end{equation}
where $\alpha=0$ leads to a constant step size.

We study the hitting time of SGLD.
For a region $\region \subseteq \mathbb{R}^d$ of interest, the hitting time to $\region$ is defined as the first time that the SGLD sequence $X_t$ falls into the region $\region$:
\begin{align}
	\tau_{\region} = \arg\min_{n \ge 0} \left\{X_n \in \region \right\} \label{eqn:hitting_time_defn}.
\end{align}


In this paper we are interested in finding approximate first order stationary points (FOSP) and second order stationary points (SOSP) of the objective function. 
Note that finding a FOSP or SOSP is sufficient for many machine learning applications. 
For convex problems, an FOSP is already a global minimum and for certain problems like matrix completion, an SOSP is a global minimum and enjoys good statistical properties as well~\citep{ge2017no}.
Please see Section~\ref{sec:rel} for discussions.

\begin{defn}[Approximate First Order Stationary Points (FOSP)]
	\label{defn:fosp}
Given $\epsilon > 0$, we define approximate first order stationary points as
	\begin{align*}
		\region_{fosp}(\epsilon) = \left\{X \in \mathbb{R}^d, \norm{\nabla F(X)} \le \epsilon\right\},
	\end{align*}
and denote by $\tau_{fosp} = \arg\min_{n\ge 0} \left\{X_n \in \region_{fosp}(\epsilon) \right\}$ the corresponding hitting time.
\end{defn} 

\begin{defn}[Approximate Second Order Stationary Points (SOSP)]
	\label{defn:sosp}
Given $\epsilon, \lambda_{\epsilon} > 0$, we define approximate second order stationary points as
\begin{align*}
		\region_{sosp}(\epsilon,\lambda_\epsilon) = \left\{X \in \mathbb{R}^d, \norm{\nabla F(X)} \le \epsilon, \lambda_{\min}\left(\nabla^2 F(X)\right) \ge -\lambda_\epsilon\right\},
\end{align*}
and denote by $\tau_{sosp} = \arg\min_{n\ge 0} \left\{X_n \in \region_{sosp}(\epsilon) \right\}$ the corresponding hitting time.
\end{defn}
We note that in Definition \ref{defn:sosp}, $\lambda_\epsilon$ is a function of $\epsilon$, and
it is often chosen as $\sqrt{\epsilon}$ in existing works~\citep{ge2017no}.
The main objective of this paper is showing that $\tau_{fosp}$ and $\tau_{sosp}$ are bounded in probability for SGLD sequence, and show how they depend on the problem parameters.

\subsection{Organization}
\label{sec:org}
This paper is organized as follows.
In Section~\ref{sec:rel}, we review related works.
In Section~\ref{sec:first_order}, we present our hitting time analysis of SGLD to the first order stationary points.
In Section~\ref{sec:second_order}, we extend our analysis to second order stationary points.
In Section~\ref{sec:ergodic}, we establish the ergodicity of SGLD.
In Section~\ref{sec:examples}, we provide three examples to illustrate the convergence  of SGLD for machine learning applications.
We conclude and discuss future directions in Section~\ref{sec:con}.
Proofs for technical lemmas and the results in Section~\ref{sec:examples} are deferred to the appendix.
%
%
%

	\section{Related Works}
	\label{sec:rel}


From the optimization point of view, \cite{raginsky2017non} gave a non-asymptotic bound showing that SGLD finds an approximate global minimizer in $\poly\left(d,1/\epsilon,1/\lambda^*\right)$ under certain conditions. 
In particular, this bound is a mixing-time bound and it depends on the inverse of the uniform spectral gap parameter $\lambda^*$ of the Langevin diffusion dynamics (Equation~\eqref{eqn:langevin_diffusion}), which is in general $O(e^{-d})$.
More recently, \citet{xu2017global} improved the dimension dependency of the bound and analyzed the finite-sum setting and the variance reduction variant of SGLD.
All these bounds  depend on the spectral gap parameter which scales exponentially with the dimension.
\citet{tzen2018local} gave a finer analysis on the recurrence time and escaping time of SGLD through empirical metastability.
While this result does not depend on the spectral gap, there is not much discussion on how does SGLD escape from saddle points. 
Therefore, finding a global minimum for a general non-convex objective with a good dimension dependence might be too ambitious.

On the other hand,  in many machine learning problems, finding an FOSP or an SOSP is sufficient.
Recently, a line of work shows FOSP or SOSP has already provided good statistical properties for achieving desirable prediction performance.
Examples include matrix factorization, neural networks, dictionary learning, e.t.c.~\citep{hardt2016identity,ge2015escaping,sun2017complete,ge2017no,park2017non,bhojanapalli2016global,du2018power,du2017spurious,ge2017learning,du2018algorithmic,mei2017solving}.

These findings motivate the research on designing provably algorithms to find FOSP and SOSP.
For FOSP, it is well known that stochastic gradient descent finds an FOSP in polynomial time~\citep{ghadimi2016accelerated} and recently improved by~\citet{allen2016variance,reddi2016stochastic} in the finite-sum setting.
For finding SOSP, \citet{lee2016gradient} showed if all SOSPs are local minima, randomly initialized gradient descent with a fixed step size  also converges to local minima almost surely. 
The classical cubic-regularization~\citep{nesterov2006cubic} and trust region~\citep{curtis2014trust} algorithms find SOSP in polynomial time if full Hessian matrix information is available.
Later, \citet{carmon2016accelerated, agarwal2016finding, carmon2016gradient} showed that the requirement of full Hessian access can be relaxed to Hessian-vector products.
When only gradient information is available, a line of work shows noise-injection helps escape from saddle points and find an SOSP~\citep{jin2017escape,du2017gradient,allen2017natasha,jin2017accelerated,levy2016power}.
If we can only access to stochastic gradient, 
\citet{ge2015escaping} showed that adding perturbation in each iteration suffices to escape saddle points in polynomial time.
The convergence rates are improved later by \citet{allen2017neon2,allen2017natasha,xu2017first,yu2017third,daneshmand2018escaping,jin2019nonconvex,fang2019sharp}.

\red{
Here, we want to emphasize these perturbed gradient methods are fundamentally different from SGLD considered in the current paper.
The noise injected in perturbed gradient methods is of order $O\left(\eta_{n}\right)$ and for SGLD, the noise is of order $O\left(\sqrt{\eta_n}\right)$, which is much bigger than $O\left(\eta_n\right)$.
While both methods can find an SOSP in polynomial time, the larger noise enables SGLD to traverse the entire state space, as we will demonstrate in Section~\ref{sec:ergodic}.
	}

Theoretically, however, there is a significant difference between SGD-based algorithms and SGLD.
In SGD, the squared norm of the noise scales quadratically with the step size, which has a smaller order than the true gradient.
On the other hand, in SGLD,  the squared norm of the noise scales linearly with the step size, which has the same order as the true gradient.
In this case, the noise in SGLD is lower bounded away from zero which enables SGLD to escape saddle points. 
Nevertheless, this escape mechanism is subtle, and it requires careful balances of hyper parameters and sophisticated analyses.
To our knowledge, \cite{zhang2017hitting} is the only work studied the hitting time of stationary points.
As we discussed in Section~\ref{sec:intro}, their analysis is indirect, which leads to loose bounds.
In this paper, we directly analyze the hitting time of SGLD and obtain tighter bounds.

Finally, beyond improving the training process for non-convex learning problems, SGLD and its extensions have also been widely used in Bayesian learning~\citep{welling2011bayesian,chen2015convergence,ma2015complete,dubey2016variance,ahn2012bayesian} and approximate sampling~\citep{brosse2017sampling,bubeck2018sampling,durmus2017nonasymptotic,dalalyan2017further,dalalyan2017theoretical,dalalyan2017user}.
These directions are orthogonal to ours because their primary goal is to characterize the probability distribution induced from SGLD.
	
	\section{Hitting Time to First Order Stationary Points}
	\label{sec:first_order}

As the first step, in this section we analyze the hitting time to FOSP by SGLD.

\subsection{Warm Up: A Continuous Time Analysis}
We will first use a continuous time analysis to illustrate the main proof idea.
Recall that if we let the step size $\eta_n \rightarrow 0$, the dynamics of SGLD can be characterized by an SDE 
\begin{align*}
d X_t = -\nabla F(X_t) dt + \delta_0 dW_t .
\end{align*}
Using Ito's formula, we obtain the dynamics of $F(X_t)$:
\[
dF(X_t)=-\|\nabla F(X_t)\|^2 dt + 
\frac12\delta_0^2 \text{tr}(\nabla^2 F(X_t)) dt + \delta_0 \langle  \nabla F(X_t),  dW_t\rangle.
\] 
Now given $\epsilon > 0$ and recall $\tau_{fosp}$ is the hitting time of the first order stationary points (FOSP). 
By  Dynkin's formula, for any $T>0$, we have
\begin{align}
0\leq \E F(X_{\tau_{fosp}\wedge T})&=F(X_0)+\E \int^{\tau_{fosp}\wedge T}_0   \left(\frac12\delta_0^2 \text{tr}(\nabla^2 F(X_t))-\|\nabla F(X_t)\|^2\right) dt.
\label{eqn:cont_first_order}
\end{align}
Note that before $\tau_{fsop}$, the gradient satisfies $\|\nabla F(X_t)\|\geq \epsilon$. 
If we assume the Hessian satisfies $\|\nabla^2 F(x)\|\leq C_2,$ then $\text{tr}(\nabla^2 F(x))\leq C_2d$.  Using the assumption that $F(\cdot)$ is non-negative, we can obtain the following estimate
\begin{align*}
F(X_0) \ge &\E \int^{\tau_{fosp}\wedge T}_0   \left(-\delta_0^2 \text{tr}(\nabla^2 F(X_t))+\|\nabla F(X_t)\|^2\right) dt \\
\ge &\E \int^{\tau_{fosp}\wedge T}_0   \left(\epsilon^2-\delta_0^2 \text{tr}(\nabla^2 F(X_t))\right) dt  \\
\ge & (\epsilon^2-C_2d \delta_0^2)\E \tau_{fosp}\wedge T.
\end{align*}
Therefore, re-arranging terms, we have
\[
\E \tau_{fosp}\wedge T\leq \frac{F(X_0)}{\epsilon^2-C_2d \delta_0^2}.
\]
Applying Markov's inequality, we have \[
\Prob(\tau_{fosp}>T)\leq \frac{F(X_0)}{T(\epsilon^2-C_2d \delta_0^2)}.
\]
The above derivations show if we pick a small $\delta_0\leq \frac{\epsilon}{\sqrt{2C_2 d}}$ and $T$ is large enough, we know SGLD hits an approximate first order stationary point in time less than $T$ with high probability.
In the next section, we use this insight from the continuous time analysis to derive hitting time bound of the discrete time SGLD algorithm.

\subsection{Discrete Time SGLD Analysis}
We first list technical assumptions for bounding the hitting time.
The first assumption is on the objective function.
\begin{aspt}
	\label{aspt:critical_obj}
	There  exists $C_2 > 0$ such that for all $x$, the objective function satisfies  $\|\nabla^2 F(x)\|\leq C_2$. 
\end{aspt}

This condition assumes  the spectral norm of the Hessian is bounded. It is a standard smoothness assumption, which guarantees the gradient descent algorithm can hit an approximate first order stationary.

Our second assumption is on the noise from the stochastic gradient.
\begin{aspt}
\label{aspt:critical_noise}
There exists $B_1 > 0, B_2 \geq 1$  such that for all $x$ and any $n \geq 0$, the gradient noise $\xi_{n+1}$ defined in \eqref{sys:GDanom} satisfies,
\[
\E_n \xi_{n+1}^T \nabla^2 F(x)\xi_{n+1}\leq B_1,\quad \E_n \|\xi_{n+1}\|^4\leq B^2_2.
\] 
In the sequel, we consider the natural filtration that describes the information up to the $n$-th iteration, 
\[
\mathcal{F}_n=\sigma\{X_i,\omega_i,\xi_i,\zeta_i,i=0,1,\cdots, n\},
\]
while $\Prob_n$ and $\E_n$ denote the conditional probability and expectation with respect to $\mathcal{F}_n$.
\end{aspt}
This assumption states that the noise has bounded moments.
Such an assumption is necessary for guaranteeing the convergence even for SGD.
Note here $\E_n \|\xi_{n+1}\|^4 \leq B^2_2$ also implies $\E_n \|\xi_{n+1}\|^2\leq B_2$ by Cauchy-Schwartz inequality.
Furthermore, using the property of the spectral norm, we know $\E_n \|\xi_{n+1}\|^2\leq B_2$ also implies $B_1 \le C_2B_2$.
Here, we explicit assume $\E_n \xi_{n+1}^T \nabla^2 F(x)\xi_{n+1}\leq B_1$ for some $B_1$ in order to exploit certain finer properties of the problem.
Now we are ready to state our first main theorem. Please recall the definition of $\eta_0$ from \eqref{sys:SGLD}, and $\eta_0$ and $\alpha$ from \eqref{sys:eta}.


\begin{thm}
	\label{thm:first}
Let  $\epsilon > 0$ be the desired accuracy and $0 < \rho <1 $ be the failure probability.  
Suppose we set $
\delta_0\leq \frac{\epsilon \sqrt{\rho}}{2\sqrt{3dC_2}},\text{ and } \eta_0\leq (6C_2)^{-1}$.
{Then there is an absolute constant $C_\alpha$ that depends only on $\alpha$ such that}
\begin{itemize}
\item for $\alpha \ge \frac{1}{2}$, if $N\geq C_\alpha\left(\frac{F(X_0)+B_1C_2\eta_0^2 }{\rho \eta_0\epsilon^2}\right)^\frac1{1-\alpha}$, or,
\item for $0<\alpha \le \frac12$, if $N\geq C_\alpha\max\left\{
\left(\frac{F(X_0)}{\rho \eta_0\epsilon^2}\right)^\frac{1}{1-\alpha},
\left(\frac{B_1C_2\eta_0}{\rho\epsilon^2}\right)^\frac{1}{\alpha}
\right\}$, or,
\item for $\alpha=0$, if $\eta_0\leq \frac{\epsilon^2\rho}{24 B_1C_2}$ and $N\geq \frac{8F(X_0)}{\epsilon^2 \rho\eta_0}$, 
\end{itemize}
we have \[
\Prob(\tau_{fsop}\geq N)\leq \rho. 
\]
\end{thm}

This theorem states that SGLD can easily hit a first order stationary point.
As compared with \cite{zhang2017hitting}, we provide an explicit hitting time estimate since we use a more direct analysis.
Note for different $\alpha$, we have different hitting time estimates.
The reason will be clear in the following proof. 
Also note that the number of iterations can be independent of the dimension, as long as $B_1$ and $C_2$ are dimension independent, and the volatility parameter $\delta_0$ uses the correct scaling. 

	
\red{It is worth noting that there is no lower bound requirement on the volatility parameter $\delta_0$. Indeed, when $\delta_0=0$, the SGLD degenerates as the SGD. We also note that SGD can find an FOSP in finite iterations (see, Section B in \cite{allen2017natasha}). The capability of finding FOSP is not a feature that is unique to SGLD.  
}


\begin{proof}[Proof of Theorem \ref{thm:first}]
Our proof follows closely to the continuous time analysis in the previous section.
Denote $\Delta=X_{n+1}-X_n=-\eta_{n+1} (\nabla F(X_n)+\xi_{n+1})+\delta_{n+1}\zeta_{n+1}$, where $\delta_{n+1}:=\delta_0\sqrt{\eta_{n+1}}$. 
This quantity corresponds to the $dX_t$ quantity in the continuous time analysis.
To proceed, we expand one iteration
	\begin{align}
	\notag
	F(X_{n+1})&=F(X_n)+\int^1_0 \langle \nabla F(X_n+s\Delta), \Delta \rangle ds\\
	\label{tmp:decom1}
	&=F(X_n)+\int^1_0 \langle \nabla F(X_n), \Delta \rangle ds+\int^1_0 ds \int^s_0 \Delta^T \nabla^2 F(X_n+t\Delta) \Delta  dt\\
	\notag
	&=F(X_n)- \eta_{n+1} \|\nabla F(X_n) \|^2+ \langle \delta_{n+1}\zeta_{n+1}-\eta_{n+1} \xi_{n+1}, \nabla F(X_n)\rangle\\
	\notag
	&\quad +\frac12 \Delta^T \nabla^2 F(X_n+\psi_{n+1}\Delta)\Delta. 
	\end{align}
Where $\psi_{n+1}$ is some number in $[0,1]$. The last term can be bounded by 
\[
\E_n\Delta^T \nabla^2 F(X_n+\psi_{n+1}\Delta)\Delta\leq  C_2\E_n\|\Delta\|^2,
\]
and furthermore by Holder's inequality,  i.e. $\E\|X+Y+Z\|^2\leq 3\E (\|X\|^2+\|Y\|^2+\|Z\|^2)$, 
\begin{align}
\notag
\E_n \|\Delta\|^2&\leq  3\eta^2_{n+1}\|\nabla F(X_n)\|^2+3\eta^2_{n+1}\E_n \|\xi_{n+1}\|^2+3\delta^2_{n+1}\E_n \|\zeta_{n+1}\|^2\\
&\leq 3\eta^2_{n+1} \|\nabla F (X_n)\|^2+3\eta_{n+1}^2 B_1 +3\delta_{n+1}^2 d. \label{tmp:second_order_perb}
\end{align}
So taking the expectation of Equation~\eqref{tmp:decom1} and plugging in Equation~\eqref{tmp:second_order_perb}, we have
	\begin{align}
	\notag
	\E_n F(X_{n+1})&\leq F(X_n)-(\eta_{n+1}-3\eta_{n+1}^2 C_2)\|\nabla F(X_n)\|^2+3C_2\eta_{n+1}^2 B_1+3C_2d\delta_{n+1}^2 \\
	\label{tmp:1step}
	&\leq F(X_n)-\frac12\eta_{n+1}\|\nabla F(X_n)\|^2+3C_2\eta_{n+1}^2 B_1+3C_2d\delta_{n+1}^2.
	\end{align}
	Summing this bound over all $n=0,\cdots,N-1$, apply total expectation, we obtain
	\begin{align*}
	\E F(X_N)&\leq F(X_0)-\frac12\E \sum^{N}_{k=1} \eta_k\|\nabla F(X_k)\|^2 +3C_2B_1\sum_{k=1}^{N}\eta_k^2 + 3C_2d\sum_{k=1}^{N} \delta_k^2 \\
	&\leq F(X_0)-\frac12\E \sum^{N}_{k=1} \eta_k\|\nabla F(X_k)\|^2 +3C_2B_1\eta_{1:N}^2+ 3C_2\delta_0^2 \eta_{1:N} d.
	\end{align*}
	Rearranging terms, recall that $F$ is nonnegative, we have
	\begin{align}
	\notag
	\E \sum^{N}_{k=1} \eta_k\|\nabla F(X_k)\|^2
	&\leq 2\left(F(X_0)-\E F(X_N)+3B_1C_2\eta_{1:N}^2+3C_2d\delta_0^2 \eta_{1:N}\right)\\
	\label{tmp:sumofgradient}
	&\leq 2\left(F(X_0)+3B_1C_2 \eta_{1:N}^2+3dC_2\delta_0^2  \eta_{1:N}\right).
	\end{align}	
Combining this bound and Markov's inequality, we have
\begin{align*}
	\Prob(\|\nabla F(X_k) \|\geq \epsilon,\, \forall k\leq N) \le & \frac{\E \sum^{N}_{k=1} \eta_k \|\nabla f(X_k)\|^2}{\eta_{1:N} \epsilon^2} \\
	\le & \frac{2F(X_0)+6B_1C_2 \eta_{1:N}^2}{\epsilon^2 \eta_{1:N}}+\frac{6C_2\delta_0^2}{\epsilon^2}d. 
\end{align*}
Lastly, note that 
\begin{equation}
\label{eqn:etasum}
\eta_{1:N}\geq \eta_0\int_1^N t^{-\alpha}dt=\eta_0O(N^{1-\alpha}),\quad
\eta^2_{1:N}\leq \eta_0^2+\eta^2_0\int_1^N t^{-2\alpha}dt=\eta_0^2O(\max\{1,N^{1-2\alpha}\}).
\end{equation}
Plugging in our choice $N$ and $\delta_0$, we have that the right hand side is smaller than $\rho$.
\end{proof}

	\section{Hitting Time Analysis to Second Order Stationary Points}
	\label{sec:second_order}
	
	\red{
Recall that without noise-injection, gradient descent cannot escape saddle points.
Since gradient descent can be viewed as a special case of SGD with $\nabla f(x,\omega)=\nabla F(x)$, it is clear that SGD in general cannot escape saddle points.
In contrast, SGLD is capable to escape saddle points and find second order stationary points.} In this section we analyze the hitting time of second order stationary points.
The key insight here is that because we add Gaussian noise at each iteration, the accumulative noise together with the negative eigenvalue of the Hessian will decrease the function value.
Again, we first use a  continuous time analysis on a simple example to illustrate the main idea.

\subsection{Warm Up: A Continuous Time Analysis for Escaping Saddle Points}
\label{sec:LDescape}
To motivate our analysis, we demonstrate how does the Langevin dynamics escape a strict saddle point. For this purpose,  we assume  $X_0=0$, and $F(x)=x^T Hx$ with $H$ being a symmetric matrix with $\lambda_{\min}(H)<0$. This example characterizes the situation when the SGLD starts at a saddle point. The resulting Langevin diffusion is actually an Ornstein Uhlenbeck (OU) process, 
\begin{equation}
\label{eqn:OU}
dX_t=-2 HX_t dt +\delta_0 dW_t. 
\end{equation}
Knowing it is an OU process, it can be written through an explicit formula,
\[
X_t= \delta_0\int^t_0e^{-2H (t-s)} dW_s. 
\]
Thus by It\^o's isometry, if $H$ has eigenvalues $\lambda_1\leq \lambda_2\leq\cdots\leq \lambda_d$,
\[
\E F(X_t)=  \delta_0^2 \int^t_0\text{tr}(e^{-4H(t-s)} H )ds= \frac {\delta_0^2} 4 \text{tr}(I-e^{-4 Ht })=\frac{\delta_0^2}4 \sum_{i=1}^d (1-e^{-4\lambda_i t})
\leq \frac {d\delta_0^2}  4- \frac{\delta_0^2}{4}e^{-4\lambda_1 t}.
\]
Since $x=0$ is a strict saddle point, i.e., $\lambda_1 < 0$,  we can pick  $t=\frac{\log 2d}{-4\lambda_1}$ to make $\E F(X_t)\leq -\frac {d\delta_0^2}  4<0$, which indicates that $X_t$ has escape the saddle point. 

\subsection{Escaping Saddle Points}
In this section, we provide theoretical justifications on why SGLD is able to escape strict saddle points.
In addition to Assumptions~\ref{aspt:critical_obj} and~\ref{aspt:critical_noise} made in Section~\ref{sec:first_order}, we also need some additional regularity conditions to guarantee that SGLD escape from strict saddle points.
\begin{aspt}
\label{aspt:saddlepoint}
We assume the following hold for the objective function $f$.
\begin{itemize}
\item There exist $C_3 \geq C_2>0$ such that for all pairs of $x,x' $, $\norm{\nabla^2 F(x) - \nabla^2 F(x')} \le C_3\|x-x'\|$.  Note that $C_2$ is defined in Assumption \ref{aspt:critical_obj}.
\item There exists $C_0$ such that $\abs{F(x)} \le C_0$ for all $x$.
\item There exists $D_4 > 0$ such that for all $x$, we assume the $\sum_{i=1}^d \lambda_i (\nabla^2 F(x))1_{\lambda_i>0}\leq D_4$.
\end{itemize}
\end{aspt}
The first assumption states that the Hessian is Lipschitz  and the second condition states the function value is bounded.
These two conditions are widely adopted in papers on analyzing how the first order methods can escape saddle points~\citep{ge2015escaping,jin2017escape}.
The third condition states that the sum of positive eigenvalues is bounded $D_4$.
Note there is a na\"ive upper bound $D_4 \le d C_2$.
However, in many cases $D_4$ can be much smaller than $d C_2$. We provide several examples in Section \ref{sec:examples}. So we take $D_4$ as a separate parameter in order to exploit more refined properties of the problem.

The following lemma characterizes the behavior of SGLD around a strict saddle point. 
It can be viewed as a discrete version of the discussions  in Section~\ref{sec:LDescape}.
\begin{lem}[Escaping saddle point]
\label{lem:escape}
Assume Assumptions ~\ref{aspt:critical_obj},~\ref{aspt:critical_noise} and~\ref{aspt:saddlepoint}. If the current iterate is $X_o$, for any fixed $q>0$ 
and $\lambda_H=\lambda_{\text{max}}(-\nabla^2 F(X_o))>0$. Denote
\[
\tau_b=\min\{n>o, \|X_n-X_o\|\geq b\},\quad b=\tfrac{q\lambda_H}{C_3}. 
\]
Denote $D_5=\frac{2}{\lambda_H}\log \frac{16D_4+40}{\lambda_H}$. There is a constant  $C_6=O(\max\{C_3q^{-1.5},\lambda_H^{-1},1, D_5\})$. So that if
\[
2D_5\leq \eta_{o+1:n}\leq 3D_5,
\]
\[
\delta_0 \leq d^{-\frac32} C_6^{-\frac72}\exp(-(9+18q) D_5\lambda_H),
\]
\[
\eta_o\leq B_2^{-1}\delta_0^2,\quad \|\nabla F(X_o)\|\leq\delta_0\sqrt{d\min\{\lambda_H,1,\eta_{o+1:n}^{-1}\}}. 
\]
Then
\[
\E_o F(X_{n\wedge \tau_b})\leq F(X_o)- \eta_{o+1:n} \delta_0^2. 
\]
\end{lem}

This lemma states that SGLD is able to escape strict saddle points in polynomial time. Its full proof  can be found in Section~\ref{sec:proof_escape_lem}.

\subsection{Hitting time of SOSP}
With Lemma~\ref{lem:escape} at hand, it is easy to derive the following hitting time result for the second order stationary point, given by Definition \ref{defn:sosp}
\begin{thm}
\label{thm:second}
Assume Assumptions~\ref{aspt:critical_obj},~\ref{aspt:critical_noise} and~\ref{aspt:saddlepoint}. For any $q>0$, let
\[
Q=2\lambda_\epsilon^{-1}\log \frac{16D_4+40}{\lambda_\epsilon},
\]
and a constant $C_6=O\left(\max\left\{q^{-1.5}C_3,\lambda_\epsilon^{-1},1,Q \right\}\right). $ Set the hyper parameters so that 
\[
\delta_0=d^{-\frac32} C_6^{-\frac72}\left(\frac{16D_4+40}{\lambda_\epsilon}\right)^{-(9+18q)},\quad \eta_0\leq \left\{B_2^{-1}\delta_0^2, C_3^{-1}\log \frac{16D_4+40}{C_3}\right\}. 
\]
Denote $\epsilon_0=\frac12\min\{\epsilon, \delta_0\sqrt{\min\{\lambda_\epsilon,1,Q^{-1}\}}\}$, then 
\[
\Prob(\tau_{sosp}>N)\leq  \frac{2C_0+\epsilon_0^2 Q}{\epsilon_0^2\eta_{1:N}}. 
\]
In other words, for any $\alpha\in [0,1)$ and $\rho>0$, SGLD will hit a second order stationary point with probability $1-\rho$,  as long the iteration number
\[
N\geq C_\alpha\left(\frac{2C_0}{\rho\eta_0 \epsilon_0^2}+\frac{Q}{\eta_0\rho}\right)^{\frac{1}{1-\alpha}}.
\] 
Here $C_\alpha$ is an abosolute constant that depends only on $\alpha$.

\end{thm}
This theorem states that SGLD is able to hit an SOSP in polynomial time, thus verifying adding noise is helpful in non-convex optimization.
This result has been established by \citet{zhang2017hitting} using an indirect approach as we discussed in Section~\ref{sec:intro}.
Our proof relies on a direct analysis.  The proof intuition is simple and similar to what is demonstrated in Section \ref{sec:LDescape}. Yet, there are three layers of technicalities. First, SGLD is an indiscretized version of the Langevin diffusion process \eqref{eqn:OU}. \red{This is the very reason why the deterministic stepsize constant $\eta_0$ is roughly the square of the stochastic volatility constant $\delta_0$. } Second, the loss function is only an approximation of $x^THx$ considered in  Section \ref{sec:LDescape}. Third, in order to apply the approximation $F(x)\approx x^T Hx$, $x$ needs to be close to the critical point. 
\red{These issues lead to delicate requirements of the stochastic volatility constant $\delta_0$: if it is too small, its strength is not enough for SGLD to escape the saddle point; if it is too large, then the followup iterates will be far away for the approximation $x^THx$ to be accurate. The additional constants in Theorem \ref{thm:second}, $Q, C_6, \epsilon_0$ and $C_\alpha$, are introduced to simplify the proofs. } We relegate the entire proof to Appendix~\ref{sec:proof_second}.


	\section{Ergodic behavior of SGLD}
	\label{sec:ergodic}

It is well known that the overdamped Langevin diffusion SDE $dX_t=-\nabla F(X_t)dt+dW_t$ is ergodic,  \citep{MSH02} 
if the loss function $F$ is coercive in the following sense:
\begin{aspt}
\label{aspt:coercive}
There are constants $c_7$ and $D_7$ such that 
\[
\|\nabla F(x)\|^2\geq c_7 F(x)-D_7,\quad \|x\|\leq c_7 F(x)+D_7, \quad \forall x. 
\]
\end{aspt}
\blue{Unlike the regularity Assumptions \ref{aspt:critical_obj}-\ref{aspt:saddlepoint}, this condition ``pushes" SGLD iterates back if their function values get too large. 
Another popular version of it is known as the dissipative condition \citep{raginsky2017non,xu2017global}. Assumption \ref{aspt:coercive} is weaker than the dissipative condition, as explained in \cite{DT20}.}

In other words, the stochastic process $X_t$ can visit any specific point in the state space with probability $1$, as long as the algorithm has run long enough.
It is also known that direct time-homogeneous discretization of overdamped Langevin diffusion, i.e. ULA, can inherit this property  \citep{MSH02}.
While SGLD is not time-homogenous or a direct discretization,  we can still show that SGLD is ergodic in the following theorem. 

\begin{thm}
\label{thm:ergodic}
Under Assumption \ref{aspt:coercive}, for any $\epsilon,p_0>0$ and point $z_0$, there is an $N$ such that 
\[
\Prob(\|X_t-z_0\|\leq \epsilon \text{ for some }t\leq N)\geq 1-p_0.
\]
An upper bound for $N$ can be established as below
 if we define the following sequence of constants
\[
M_V:=\frac{8}{c_7}\left( D_7+6 B_1+6d\delta_{0}^2\right),\quad D_X:=\max\{D_7+ c_7 M_V,\|z_0\|,1\}, 
\]
\[
D_F=\max\{\|\nabla F(x)\|: \|x\|\leq 4D_X\},\quad \epsilon_0:=\min\left\{\frac{\epsilon}{2D_F+2\sqrt{B_2}+1},\frac{D_F}{D_X}\right\},
\]
\[
n_0=\min\{t: \eta_t\leq \epsilon_0\}= \left(\frac{\epsilon_0}{\eta_0}\right)^\frac{1}{\alpha},
\]
\[
c_\alpha=\frac14\Prob\left(\|Z-\tfrac{2D_X }{\delta_0\sqrt{\epsilon_0}}e_1\|\leq \tfrac{\epsilon_0}{\delta_0\sqrt{2\epsilon_0}}\right),
\]
where $Z\sim N(0,I_d)$ and $e_1=[1,0,0,\ldots,0]$ is the first Euclidean basis vector.
with  Then 
\[
N\leq n_0+\left(\frac{\left\lceil\tfrac{2\log \frac12 M_V /\E F(X_{n_0})}{c_7 \epsilon_0} \right\rceil+\lceil\frac{\log \frac12p_0}{\log(1-c_\alpha)}\rceil (1+\tfrac{8}{c_7\epsilon_0})}{4\delta_0p_0\eta_0}\right)^{\frac{1}{1-\alpha}}.
\]
\end{thm}
\red{The proof of Theorem \ref{thm:ergodic} follows a similar framework as in \cite{MT93, MSH02}, which consists of two parts. In the first part, we show the compact sub-level set $\{x: F(x)\leq M_V\}$ will be visited by SGLD infinitely many times due to Assumption \ref{aspt:coercive}. In particular, it will take 
\[
N_0\leq n_0+\left(\frac{\left\lceil\tfrac{2\log \frac12 M_V /\E F(X_{n_0})}{c_7 \epsilon_0} \right\rceil}{4\delta_0p_0\eta_0}\right)^{\frac{1}{1-\alpha}}
\]
iterations to guarantees $\E F(X_{N_0})\leq M_V$, when $n_0$ is the first iteration and the step size is small enough. } 
\red{In the second part, we show that if $F(x_o)\leq M_V$, then after a certain number of iterations $n_o$, $x_{o+n_o}$ has a positive chance, bounded below by $c_\alpha$, to hit $z_0$. This can be done using stochastic controllability of the process.  To combine this with the first part, we set up a Bernoulli-trial type of argument, where the trial number $\lceil\frac{\log \frac12p_0}{\log(1-c_\alpha)}\rceil$ shows up naturally.
It is also worth noting that since $z_0$ is an arbitrary point in the state space, the trial-success probability $c_\alpha$ is often pessimistic and scales exponentially with the dimension. This is in sharp contrast with Theorem \ref{thm:first} and \ref{thm:second}. } 

\blue{Ergodicity is a feature of SGLD. It can be useful when the objective function $F$ has multiple local minimums, as SGLD can visit all of them and find the global minimum.
This is one general motivation of injecting randomness in algorithms, since  GD with random initialization may not converge to the local minimums in polynomial iterations~\citep{du2017gradient}. 
While Perturbed (S)GD can be ergodic with constant stepsize, it is not ergodic with decreasing stepsizes. To see this, we consider a simple test case, where we can show SGLD  converges to $\mathcal{N}(0,1/2)$ and PGD converges to the origin. 
Note that while PGD does converge to the correct global minimum in this example, it is not ergodic and  does not explore the state space in long term. 
\begin{lem}
\label{lem:x2}
Suppose
\[
F(x)=\frac12x^2,\quad\xi_{n}=0,\quad\delta_0=1,\quad\zeta_{n}\sim \mathcal{N}(0,1).
\]
The SGLD and PGD updates are given by
\[
X_{n}=(1-\eta_n)X_{n-1}+\sqrt{\eta_{n}}\zeta_{n},\quad Y_{n}=(1-\eta_n)Y_{n-1}+\eta_{n}\zeta_{n}. 
\]
Assume that $X_0=Y_0=0$, the stepsize decreases to zero and $\sum_{i=1}^\infty\eta_i=\infty$.
The limiting distribution of $X_n$ is $\mathcal{N}(0,1/2)$ and the limiting distribution of $Y_n$ is the Dirac measure at $0$.
\end{lem}}

%
%


	\section{Applications in Online Estimation Problems}
	\label{sec:examples}
In the previous sections, we have shown that, as long as the objective function satisfies some smoothness conditions and the noise satisfies certain moment conditions, then SGLD hits a first/second order stationary point in polynomial time in terms of the following parameters:
\begin{enumerate}
	\item $\|\nabla^2 F(x)\|\leq C_2$. 
	\item $\E_n \xi_{n+1}^T \nabla^2 F(x)\xi_{n+1}\leq B_1$. 
	\item $\E_n \|\xi_{n+1}\|^2\leq B_2$, $\E_n \|\xi_{n+1}\|^4\leq B^2_2$.
	\item  $\sum_{i=1}^d \lambda_i (\nabla^2 F(x))1_{\lambda_i>0}\leq D_4$.
	\item $\|\nabla^2 F(x)-\nabla^2 F(x')\|\leq C_3\|x-x'\|$.
\end{enumerate}
In particular, first order stationary point hitting time bound only relies on the first two constants. And to verify Assumption \ref{aspt:coercive}, we need 
\[
\|\nabla F(x)\|^2\geq c_7 F(x)-D_7,\quad \|x\|\leq c_7 F(x)+D_7, \quad \forall x. 
\]
In this section we provide three concrete example problems, linear regression, online matrix sensing, and online PCA to illustrate the usage of our analysis of SGLD. 
We will calculate the specific problem dependent constants in Assumption~\ref{aspt:critical_obj} and Assumption~\ref{aspt:saddlepoint}.
Note that once we know these constants are bounded,  Theorem~\ref{thm:first} and Theorem~\ref{thm:second} directly imply polynomial time hitting time.
For all examples, we investigate the stochastic optimization setting, i.e., each random sample is only used once. All the proofs are deferred to Appendix~\ref{sec:proof_examples}.

When calculating the constants defined above, they often have a positive dependence on the norm of the location $x$ where they are evaluated. Therefore we will assume in below that the iterate $x$ is bounded. This assumption sometimes is assumed even in the general theoretical analysis for simplicity \cite{zhang2017hitting}. We remark this is not a restrictive assumption. Since in practice, if the SGLD iterates diverge to infinity for a particular application, it is a clear indication that the algorithm is not fit for this application. 


\subsection{Linear regression}
\label{sec:lr}
Our first example is the classical linear regression problem.
Let the $n$-th sample be $\omega_n = (a_n, b_n)$, where the input $a_n \in \mathbb{R}^d$ is a sequence of random vectors independently drawn from the same distribution $\mathcal{N}(0, A)$, and the response $b_n \in \mathbb{R}^d$ follows a linear model,
\begin{equation}\label{eq:linear_model}
  b_n = a_n^T x^* + \varepsilon_n.
\end{equation}
Here, $x^* \in \mathbb{R}^d$ represents the true parameters of the linear model, and $\{\varepsilon_n\}$ are independently and identically distributed (\emph{i.i.d.}) $\mathcal{N}(0,1)$ random variables, which are uncorrelated with $a_n$. For simplicity, we assume $a_n$ and $\varepsilon_n$ have all moments being finite.
We consider the quadratic loss, i.e.,  given $\omega_n = (a_n, b_n)$, the loss function is 
\[
f(x, \omega_n)=\frac{1}{2} (a_n^T x -b_n)^2.
\]
and the population loss function is defined as 
\begin{equation}
\label{eqn:poplin}
F(x)=\frac12(x-x^*)^TA(x-x^*)+\frac12. 
\end{equation}
Note the underlying linear coefficient satisfies $x^*=\arg\min_x \left(F(x):=\E f(x, \omega)\right)$.
The following proposition bounds the constants in Assumption~\ref{aspt:critical_obj} and~\ref{aspt:saddlepoint}.
\begin{prop}
	\label{prop:lr}
When $\|x-x^*\|\leq \Gamma$, the constants for the population loss function in \eqref{eqn:poplin} are give by 
\[
D_4= C_2=C_3= \tr(A),\quad B_1=c\tr(A)^3\Gamma^2,\quad B_2= c\tr(A)^2 \Gamma^2. 
\]
Here, $c>0$ is a constant independent of other parameters. Moreover, we have
\[
c_7=\frac14\lambda_{\min}(A),\quad D_7\leq \frac14\lambda_{\min}(A)\|x^*\|^2.
\]
\end{prop}
In other words, the number of iterations for the SGLD to hit a first order stationary point only depends on $\tr(A)$, but not directly on the problem dimension $d$. When $\|A\|$ is fixed, $\tr(A)$ depends on the rank of $A$, which can be much smaller than $d$.  Moreover, when the spectrum of $A$ is a summable sequence (e.g.,  $\lambda_n(A)\propto n^{-\beta}$ where $\beta>1$ and $\lambda_n(A)$ is the $n$-th largest eigenvalue of $A$), then $\tr(A)$ is  independent of $d$. Such settings rise frequently in Bayesian problems related to partial differential equations, where $d$ in theory can be infinity \citep{st10}.

\subsection{Matrix Factorization}
\label{sec:mf}
In the online matrix factorization problem, we want to minimize the following function
\begin{equation}
\label{eqn:popmatf}
 F(X)=\| XX^T-M\|^2_F,\quad X\in \reals^{m\times r},
\end{equation}
where $M$ has rank $r$ and $\|\,\cdot\,\|_F$ denotes the Frobenius norm.  The stochastic version is given by 
\[
f(X,\omega)=\langle \omega, XX^T-M\rangle^2.
\]
$\omega$ is assumed to be an $m\times m$ matrix with \emph{i.i.d.} entries from $\mathcal{N}(0,1)$.
The following proposition bounds the constants in Assumption~\ref{aspt:critical_obj} and~\ref{aspt:saddlepoint}.

\begin{prop}
	\label{prop:mf}
Suppose $\|M\|_F, \|XX^T\|_F\leq \Gamma$ for some constant $\Gamma\geq 1$, the constants for the population loss function \eqref{eqn:popmatf} is bounded by 
\[
C_2=24\Gamma, \quad C_3=\max\{12\sqrt{\Gamma},24\Gamma\},\quad D_4=(4m+2r)\Gamma,\quad B_1=cmr\Gamma^4,\quad B_2= cmr \Gamma^3. 
\]
Here, $c$ is a constant independent of other parameters. Moreover
\[
c_7\geq 4\lambda_{\max}(M),\quad D_7\leq \max\{\|M\|_F, 8d\lambda^3_{\max}(M) \}. 
\]
\end{prop}
Similar to the linear regression case, if $M$ is low rank or its spectrum decays rapidly to zero, the number of iterations for the SGLD to hit a first  order stationary point is independent of the problem dimension.

\subsection{Online PCA}
\label{sec:pca}
In the online principle component analysis (PCA) problem, we consider the scenario where we conduct PCA for data samples $x_i\sim \mathcal{N}(0,M)$. The population loss function is given by 
\begin{equation}
\label{eqn:popPCA}
 F(X)=\frac1{2}\| XX^T-M\|^2_F+C,\quad X\in \reals^{d\times r},
\end{equation}
where  $C=\E x_i x_i^T$.  The stochastic version is given by,
\[
f(X,\omega_i)=\frac1{2}\| XX^T-x_i x_i^T\|^2_F=\frac1{2}\| XX^T-M-\omega_i\|^2_F,\quad \omega_i=x_ix_i^T-M. 
\]
The following proposition bounds the constants in Assumption~\ref{aspt:critical_obj} and~\ref{aspt:saddlepoint}.
\begin{prop}
\label{prop:pca}
Suppose $\|M\|_F, \|XX^T\|_F\leq \Gamma$ for some constant $\Gamma\geq 1$, the constants for the population loss function \eqref{eqn:popPCA} is bounded by 
\[
C_2=24\Gamma, C_3=12\sqrt{\Gamma}, \quad D_4=6m\Gamma,\quad B_1=c\Gamma^2 \tr(M)^2,\quad B_2= \tr(M)^2. 
\]
Moreover
\[
c_7\geq 4\lambda_{\max}(M),\quad D_7\leq \max\{\|M\|_F, 8d\lambda^3_{\max}(M) \}. 
\]
\end{prop}
Since the hitting time to a first order stationary point only depends on $C_2$ and  $B_1$,  Proposition \ref{prop:pca} shows that the SGLD hits a first order stationary point with the number of iterations independent of the problem dimension.

	\section{Conclusion and Discussions}
	\label{sec:con}
In this paper we present a direct analysis for hitting time of  SGLD for the first and second order stationary points.
Our proof only relies on basic linear algebra,  and probability theory.
Through this directly analysis, we show how different factors, such as smoothness of the objective function, noise strength, and step size, affect the final hitting time of SGLD.
We also present three examples, online linear regression, online matrix factorization, and online PCA, which demonstrate the usefulness of our theoretical results in understanding SGLD for stochastic optimization tasks.
An interesting future direction is to extend our proof techniques to analyze SGLD for optimizing deep neural networks.
We believe combing recent progress in the landscape of deep learning~\citep{yun2018critical,kawaguchi2016deep,du2018power,hardt2016identity}, this direction is promising.

	\section*{Acknowledgement}
	The research of Xi Chen is supported by NSF Award (Grant No. IIS-1845444). 	
	The research of	Simon S. Du is supported by National Science Foundation (Grant No. DMS-1638352) and the Infosys Membership.	
	The research of Xin T. Tong is supported by the and Singapore MOE grants R-146-000-258-114 and R-146-000-292-114. The authors would like to thank the action editor and referees for their valuable suggestions and comments, which greatly improves this paper. We also changed the title of the paper to reflect our new contribution on the ergodicity of SGLD.
	The authors also thank Qiang Liu and Zhuoyi Yang for proofreading of the manuscript. 
\newpage
	
	\appendix
	
	\section{Technical Proofs of SOSP hitting analysis}
	\label{sec:proof_second}
In this section, we provide technical verifications of our claims made in Section \ref{sec:second_order}.

\subsection{Taylor expansions and preliminary lemmas}
Our analysis relies on Taylor expansions near critical points. To this end, for any given iterate $X_o\in \Omega$, denote 
\[
v_o:=\nabla F(X_o), \quad H=\nabla^2 F(X_o),\quad \lambda_H=\lambda_{\text{max}}(-\nabla^2 F(X_o)). 
\]
Assumption \ref{aspt:saddlepoint} leads to the following expansion
\begin{align*}
F(x)=&F(X_o)+v_o^T (x-X_o)+(x-X_o)^TH(x-X_o)+R_0(x-X_o),\quad R_0(x-X_o)\leq C_3\|x-X_o\|^3, \\
\nabla F(x)=&\nabla F(X_o)+H(x-X_o)+R_1(x-X_o),\quad  R_1(x-X_o)\leq C_3\|x-X_o\|^2,\\
\nabla^2 F(x)=&H+R_2(x-X_o),\quad \|R_2(x-X_o)\|\leq C_3\|x-X_o\|. 
\end{align*}
where $R_1$, $R_2$ and $R_3$ are reminder terms of Taylor expansion.

\label{sec:proof_escape_lem}
In order to apply the Taylor expansion, it is necessary for the subsequent iterates to be close to $X_o$. To this end, we set up the following a-priori upper bound.
\begin{lem}
\label{lem:Xdev}
Assume Assumptions ~\ref{aspt:critical_obj},~\ref{aspt:critical_noise} and~\ref{aspt:saddlepoint}. Suppose $\lambda_H>0$. Consider the $o$-th to $n$-th iterations, where the step sizes satisfy 
\[
\eta_n\leq \eta_o\leq \min\{(100C_3)^{-1}, \tfrac19\}.
\] 
Then we have the following deviation bound for any $1>q>0$, $Q=5q^{-1}\lambda^{-1}_H$
\[
\E_{o} \| X_{n\wedge\tau_b}-X_o\|^2\leq 5q^{-1}\lambda_H^{-1}\exp((2+4q)\eta_{o+1:n} \lambda_H)\eta_{o+1:n}(\delta_0^2 d+\lambda_H^{-1}\|v_o\|^2+\eta_o B_2).
\]
\[
\E_o \| X_{n\wedge \tau_b}-X_o\|^4\leq  250q^{-3}\lambda_H^{-4}\exp((4+12q)\eta_{o+1:n} \lambda_H) \eta_{o+1:n}(\lambda_H^{-2} \|v_o\|^4+\delta_o^4 d^2+\eta^2_o B_2^2). 
\]
Here with a constant $b\leq \tfrac{q\lambda_H}{C_3}$, $\tau_b$ is the stopping time of moving distance of at least  $b$:
\[
\tau_b:=\inf\{t: \|X_n-X_o\|\geq b\},  
\]
Therefore by Markov inequality, 
\[
\Prob_o(\tau_b\leq n)\leq\frac{\E_o \| X_{n\wedge \tau_b}-X_o\|^4}{b^4} \leq  250q^{-7}\exp((4+12q)\eta_{o+1:n} \lambda_H) \eta_{o+1:n}( \lambda_H^{-2}\|v_o\|^4+\delta_0^4 d^2+\eta^2_o B_2^2)C_3^4\lambda_H^{-5}. 
\]
\[
\sqrt{\Prob_o(\tau_b\leq n)}\leq 17q^{-3.5}\exp((2+6q)\eta_{o+1:n} \lambda_H) \sqrt{\eta_{o+1:n}}(\lambda_H^{-1} \|v_o\|^2+\delta_0^2 d+\eta_o B_2)\lambda_H^{-2.5}. 
\]
\end{lem}

\begin{proof}
Denote $\Delta X_n=X_n-X_o$, it follows the recursion below,
\begin{equation}
\label{tmp:saddledec1}
\Delta X_{n+1}=\Delta X_n-\eta_{n+1}(v_o+H\Delta X_n+R_1(\Delta X_n))-\eta_{n+1}\xi_{n+1}+\delta_{n+1} \zeta_{n+1}.
\end{equation}
Note that if $\tau_b> n$, $\|\Delta X_{n}\|\leq b$, so 
\begin{align*}
\|\Delta X_{n}-\eta_{n+1}(v_o+H\Delta X_n+R_1(\Delta X_n))\|
&\leq (\|I-\eta_{n+1} H\|+C_3 b\eta_{n+1})\|\Delta X_n\|+\eta_{n+1}\|v_o\|\\
&\leq (1+(1+q)\lambda_H\eta_{n+1})\|\Delta X_n\|+\eta_{n+1}\|v_o\|.
\end{align*}
We take square of this bound, note that $\lambda_H\leq C_2\leq C_3$, the following estimate holds because of Young's inequality
\begin{align}
\notag
&\|\Delta X_n-\eta_{n+1}(v_o+H\Delta X_n+R_1(\Delta X_n))\|^2\\
\notag
&\leq (1+2(1+q)\lambda_H\eta_{n+1}+(1+q)^2\lambda_H^{2}\eta^2_{n+1})\|\Delta X_n\|^2+4\eta_{n+1}\|v_o\|\|\Delta X_n\| +\eta_{n+1}^2\|v_o\|^2\\ 
\notag
&\leq (1+(2+3q)\lambda_H\eta_{n+1})\|\Delta X_n\|^2+ 4\eta_{n+1}\|v_o\|\|\Delta X_n\|+\eta_{n+1}^2\|v_o\|^2\\
\notag
&\leq (1+(2+3q)\lambda_H\eta_{n+1})\|\Delta X_n\|^2+ (q\lambda_H\eta_{n+1} \|\Delta X_n\|^2+4q^{-1}\lambda_H^{-1}\eta_{n+1} \|v_o\|^2)+\eta_{n+1}^2\|v_o\|^2\\
\label{tmp:l2step1}
&\leq (1+(2+4q)\lambda_H\eta_{n+1})\|\Delta X_n\|^2+Q\eta_{n+1} \|v_o\|^2. 
\end{align}
Here we used that $\lambda_H^{-1}\eta_{n+1}\geq C_3^{-1}\eta_{n+1}\geq \eta_{n+1}^2$ and we let $Q:=5q^{-1}\lambda_H^{-1}$. Combine this with \eqref{tmp:saddledec1} we can conclude that, if $\tau_b>n$
\begin{align*}
\E_n\|\Delta X_{n+1}\|^2&=\|\Delta X_n-\eta_{n+1}(v_o+H\Delta X_n+R_1(\Delta X_n))\|^2+\eta_{n+1}^2B_2+\delta_{n+1}^2d\\
&\leq (1+(2+4q)\eta_{n+1} \lambda_H)\|\Delta X_n\|^2+\eta_{n+1} (\eta_oB_2+\delta_0^2 d+Q\|v_o\|^2). 
\end{align*}
Therefore we have
\[
\E_n\|\Delta X_{\tau_b\wedge n+1}\|^2\leq (1+(2+4q)\eta_{n+1} \lambda_H)\|\Delta X_{\tau_b\wedge n}\|^2+\eta_{n+1} (\eta_oB_2+\delta_0^2 d+Q\|v_o\|^2),
\]
because it is trivial to verify  if $\tau_b\leq n$. Next, because $\E_o\|\Delta X_{\tau_b\wedge o}\|^2=0$, by Gronwall's inequality 
\[
\E_o \|\Delta X_{n\wedge \tau_b}\|^2\leq \eta_{o+1:n} \exp((2+4q)\eta_{o+1:n} \lambda_H)(\delta_0^2 d+Q\|v_o\|^2+\eta_o  B_2).
\] 
Likewise, we can bound the fourth moment by taking square of \eqref{tmp:l2step1}
\begin{align*}
&\|\Delta X_n-\eta_{n+1}(v_o+H\Delta X_n+R_1(\Delta X_n))\|^4\\
&\leq (1+(2+4q)\lambda_H\eta_{n+1})^2\|\Delta X_n\|^4+2Q\eta_{n+1}(1+(2+4q)\lambda_H\eta_{n+1})\|\Delta X_n\|^2\|v_o\|^2+Q^2\eta^2_{n+1}\|v_o\|^4\\
&\leq 
(1+q\lambda_H\eta_{n+1})(1+(2+4q)\lambda_H\eta_{n+1})^2 \|\Delta X_n\|^4+q^{-1}Q^2\lambda_H^{-1}\eta_{n+1}\|v_o\|^4+Q^2\eta^2_{n+1}\|v_o\|^4\\
&\leq (1+(4+10q)\lambda_H\eta_{n+1})\|\Delta X_n\|^4+Q^3\eta_{n+1} \|v_o\|^4. 
\end{align*}
Therefore, by Young's inequality, if $\tau_b>n$
\begin{align*}
\E_n\|\Delta X_{n+1}\|^4&=\|\Delta X_n-\eta_{n+1}(v_o+H\Delta X_n+R_1(\Delta X_n))\|^4+(\E \eta^4_{n+1}\|\xi_{n+1}\|^4+ 3\delta_{n+1}^4d^2)\\
&\quad+2\|\Delta X_n-\eta_{n+1}(v_o+H\Delta X_n+R_1(\Delta X_n))\|^2(\E \eta^2_{n+1}\|\xi_{n+1}\|^2+\delta_{n+1}^2d)\\
&\leq \|\Delta X_n-\eta_{n+1}(v_o+H\Delta X_n+R_1(\Delta X_n))\|^4+(\E \eta^4_{n+1}\|\xi_{n+1}\|^4+ 3\delta_{n+1}^4d^2)\\
&\quad+q\lambda_H\eta_{n+1}\|\Delta X_n-\eta_{n+1}(v_o+H\Delta X_n+R_1(\Delta X_n))\|^4\\
&\quad+q^{-1}\lambda^{-1}_H(\E \eta^2_{n+1}\|\xi_{n+1}\|^2+\delta_{n+1}^2d)^2/\eta_{n+1}\\
&\leq (1+q\lambda_H \eta_{n+1})(1+ (4+10q)\lambda_H\eta_{n+1})\|\Delta X_n\|^4+(1+q\lambda_H \eta_{n+1})\eta_{n+1}Q^3\|v_o\|^4\\
&+ Q\eta_{n+1}(\eta^2_{o}B_2^2+\delta_0^4 d^2)\\
&\leq (1+ (4+12q)\lambda_H\eta_{n+1})\|\Delta X_n\|^4+2Q^3 \lambda_H^{-1}\eta_{n+1}( \lambda_H^{-2}\|v_o\|^4+\eta^2_{o}B_2^2+\delta_0^4 d^2)
\end{align*}
So we conclude that 
\[
\E_n\|\Delta X_{\tau_b\wedge n+1}\|^4\leq (1+ (4+12q)\lambda_H\eta_{n+1})\|\Delta X_{\tau_b\wedge n}\|^4+ 2Q^3\lambda_H^{-1}\eta_{n+1}( \lambda_H^{-2}\|v_o\|^4+\eta^2_{o}B_2^2+\delta_0^4 d^2),
\]
because it holds trivially when $\tau_b<n$.
Using Gronwall's inequality, 
\[
\E_o \|\Delta X_{n\wedge \tau_b}\|^4\leq  2Q^3\lambda_H^{-1}\eta_{o+1:n}\exp((4+12q)\eta_{o+1:n} \lambda_H) ( \lambda_H^{-2}\|v_o\|^4+\delta_0^4 d^2+\eta_o^2B_2^2). 
\]
\end{proof}

To continue, we need a series of estimates of matrix products. The following lemma collects the results we need.
\begin{lem}
\label{lem:Hbounds}
If $H$ is a $d\times d$ real symmetric matrix with $\lambda_H=\lambda_{max}(-H)>0$. Let
\[
A_{j:k}=(I-\eta_{j+1} H)\cdots (I-\eta_k H)
\]
Suppose $\eta_i\|H\|<\frac12$ and $\sum_{i=1}^d \lambda_i (H)1_{\lambda_i(H)>0}\leq D_4$ We have the following estimates,
\begin{enumerate}[a)]
\item $\| A_{j:k}\|=(1+\eta_{j+1} \lambda_H)\cdots (1+\eta_k \lambda_H)\leq \exp(\lambda_H\eta_{j+1:k})$.
 \item $ \sum_{j=o+1}^{n} \eta_j\| A_{j:n}\|=\|(I-A_{o:n})H^{-1}\|\leq\eta_{o+1:n}\exp(\lambda_H\eta_{o+1:n})$. 
 \item $( I-A_{o:n}) H^{-1} A_{o:n}\succeq 0$. And if $\eta_{o+1:n}\geq C_3^{-1}e^{-\eta_{o+1:n} C_3}, ( I-A_{o:n}) H^{-1} A_{o:n}\succeq \frac12 (C_3^{-1}e^{-\eta_{o+1:n} C_3} )I$.
\item $\sum_{j=o+1}^n \eta_j \text{tr} A_{j:n} H A_{j:n}\leq -\frac14(\exp(\lambda_H \eta_{o+1:n})-1)+2\eta_{o+1:n}D_4 .$
\item If $\eta_{o+1:n}\geq \frac{2}{\lambda_H}\log \frac{16D_4+40}{\lambda_H}$, then 
\[
 \sum_{j=o+1}^n \delta_0^2\eta_j \text{tr} A_{j:n} H A_{j:n}+\eta^2_j \lambda_{\max} (A_{j:n} H A_{j:n})\leq  -4\delta_0^2\eta_{o+1:n}.
 \]
 \end{enumerate}

 \end{lem}
\begin{proof}
Since $H$ is real symmetric, and $A_{j:k}$ is a polynomial of $H$, the proofs goes by checking the eigenvectors. Let $-\lambda_H=\lambda_1<\cdots<\lambda_d$ be the eigenvalues of $H$, and $v_i$ be their eigenvectors with norm being 1. Then 
\[
A_{j:k}v_i=(1-\eta_{j+1} \lambda_i)\cdots (1-\eta_{k} \lambda_i) v_i.
\]
Since we assume that $|\eta_i \lambda_i|<1$, so  claim a) is straightforward.

Note that $A_{j-1,n}+\eta_jH A_{j:n}=A_{j:n}$, $A_{n:n}=I$,  so $\sum_{j=1}^n \eta_j HA_{j:n}=I-A_{o:n}$. Because $A_{j:n}$ is a polynomial of $H$, it commutes with $H$, so
$\sum_{j=1}^n \eta_j A_{j:n}H=I-A_{o:n}$ as well.  Note that $(I-A_{o:n})H^{-1}$ is a polynomial of $H$, so it is always well defined even if $H$ is singular. Therefore, 
\begin{align*}
\|(I-A_{o:n})H^{-1}\|&= \sum_{j=o+1}^{n} \eta_j \|A_{j:n}\|\leq\eta_{o+1:n} \max_{o+1\leq j\leq n} \|A_{j:n}\|\leq \eta_{o+1:n}\exp(\lambda_H\eta_{o+1:n}). 
\end{align*}

As for claim c), it is easy to check that for any uni-norm eigenvector $v_i$, 
\[
v_i^T(I-A_{o:n}) H^{-1} A_{o:n}v_i= \lambda_i^{-1}(1-a_i)a_i>0, \quad a_i:=(1-\lambda_i \eta_{o+1})\cdots (1-\lambda_i \eta_{n}). 
\]
We only need to show $ \lambda_i^{-1}(1-a_i)a_i\geq \frac12(C_3^{-1}e^{-\eta_{o+1:n} C_3})$.

When $\lambda_i\leq 0$, $a_i$ will be larger than $1$, so
\[
\lambda_i^{-1}(1-a_i)a_i\geq \lambda_i^{-1}(1-a_i)\geq \eta_{o+1:n},
\]
and we can get to our claim c). We only need to concern the case that $\lambda_i>0$. If $a_i>\frac12$, since $(1-\lambda_i \eta_{o+1})(1-\lambda_i \eta_{o+2})\geq (1-\lambda_i \eta_{o+1:o+2})$
\[
\lambda_i^{-1}(1-a_i)a_i\geq \frac12\lambda_i^{-1}(1-a_i)\geq \frac12\lambda_i^{-1}(1-(1-\lambda_i \eta_{o+1:n}))\geq \frac12\eta_{o+1:n}.
\]
This is larger than our claimed lower bound. When $a_i\leq \frac12$, note that $a_i\geq \exp(-2\lambda_i\eta_{o+1:n})$, so 
\[
\lambda_i^{-1}(1-a_i)a_i\geq \frac12\lambda_i^{-1} a_i\geq  \frac12C_3^{-1} \exp(-2C_3\eta_{o+1:n}).
\]

As for claim d), note that $v_i^TA_{j:n} HA_{j:n} v_i$ shares the same sign with $\lambda_i$, in particular it is negative if $\lambda_i<0$.

If $\lambda_i\geq 0$, then 
\[
v_i^TA^2_{j:n}v_i-v_i^TA^2_{j-1:n}v_i=(2\eta_j \lambda_i -\eta_j^2 \lambda_i^2)v_i^TA^2_{j:n}v_i\geq \eta_j v_i^TA_{j:n} HA_{j:n} v_i.
\]
Summing over both sides for $j=o+1$ to $n$, we find
\begin{align*}
\sum_{j=o+1}^n\eta_j v_i^TA_{j:n} HA_{j:n} v_i\leq 1-v_i^T A^2_{o:n}v_i\leq 1-\exp(-2\eta_{o+1:n} \lambda_i).
\end{align*}
If $\lambda_i<0$, 
\[
v_i^TA^2_{j:n}v_i-v_i^TA^2_{j-1:n}v_i=(2\eta_j \lambda_i -\eta_j^2 \lambda_i^2)v_i^TA^2_{j:k}v_i\geq 4\eta_j v_i^TA_{j:n} HA_{j:n} v_i.
\]
\begin{align*}
\sum_{j=o+1}^n\eta_j v_i^TA_{j:n} HA_{j:n} v_i\leq \frac14-\frac14v_i^T A^2_{o:n}v_i\leq \frac14-\frac14 \exp(-\eta_{o+1:n}\lambda_i)\leq 0. 
\end{align*}
Here we used that $(1+|\lambda_i \eta_j|)\geq \exp(\tfrac12 |\lambda_i \eta_j|)$ when $|\lambda_i \eta_j|<\frac12$, which is true under our assumption of step sizes.
Note that $-\exp(-x)$ is a concave function, so by Jensen's inequality, for any $m$, 
\[
\sum_{i=1}^m (1-\exp(-x_i))\leq m-m\exp\left(-\frac{\sum_{i=1}^mx_i}{m}\right).
\]
Using this, assuming there are $d_+$ positive eigenvalues, and their sum is below $D_4$, we find that
%
\begin{align}
\notag
\sum_{j=o+1}^n\eta_j &\text{tr}(A_{j:n} HA_{j:n}) =\sum_{i=1}^d\sum_{j=o+1}^n\eta_j v_i^TA_{j:n} HA_{j:n} v_i\\
&=\sum_{j=o+1}^n\eta_j v_1^TA_{j:n} HA_{j:n} v_1+\sum_{i: \lambda_i\geq 0}^d\sum_{j=o+1}^n\eta_j v_i^TA_{j:n} HA_{j:n} v_i
\label{tmp:AHA1}\\
\notag
&\leq \frac14(1-\exp(\lambda_H \eta_{o+1:n}))+\sum_{i: \lambda_i\geq 0}^d1-\exp(-2\eta_{o+1:n} \lambda_i)\\
\notag
&\leq \frac14(1-\exp(\lambda_H \eta_{o+1:n}))+d_+\left(1-\exp(-\tfrac{2\eta_{o+1: n} D_4}{d_+})\right)\\
\notag
&\leq \frac14(1-\exp(\lambda_H \eta_{o+1:n}))+2\eta_{o+1: n} D_4.
\end{align}
 Claim e) is a consequence of claim d).  Note that 
 \[
\sum_{j+1}^n\eta_j^2\max_i v_i^T A_{j:n} H A_{j:n} v_i\leq \eta_o\sum_{j=o+1}^n\sum_{i: \lambda_i\geq 0}^d\eta_j v_i^TA_{j:n} HA_{j:n} v_i. 
 \]
This is the same as the second part of \eqref{tmp:AHA1}, we find it is bounded above by $2\eta_o\eta_{o+1:n} D_4\leq 2\delta_0^2\eta_{o+1:n} D_4$. In other words, we find that
\[
\sum_{j+1}^n\delta_0^2\eta_j \text{tr}(A_{j:n} HA_{j:n})+\eta_j^2\max_i v_i^T A_{j:n} H A_{j:n} v_i
\leq \frac{\delta_0^2}4(1-\exp(\lambda_H \eta_{o+1:n}))+4\eta_{o+1: n}\delta_0^2 D_4.
\]
Note that for any $x$, by convexity of $e^{\lambda x}$
 \[
 e^{\lambda x}-1=\lambda \int^x_0 e^{\lambda y}dy\geq \lambda \int^x_0 e^{\frac12\lambda x}dy=\lambda x e^{\frac12\lambda x}. 
 \]
So it suffices to show that 
 \[
 \lambda_H \exp(\frac12 \lambda_H \eta_{o+1:n})\geq 16D_4+40.
 \]
 Taking log we have
 \[
 \eta_{o+1:n}\geq \frac{2}{\lambda_H}\log \frac{16D_4+40}{\lambda_H},
 \]
 which is our condition.
 
\end{proof}

A crucial step in the proof is a careful analysis of the noise and the reminder terms.
In the next lemma, we give estimates for these perturbations.
\begin{lem}
\label{lem:Urbound}
Denote $A_{j:k}=\prod^{k}_{i=j+1} (I-\eta_i H)$,  and $\Delta X_n=X_n-X_o$, then it admits the following decomposition,
\[
\Delta X_{n}=U_n+r_n, 
\]
where
\[
U_n:=\sum_{j=o+1}^n A_{j:n} (-\eta_{j}v_o+\delta_{j} \zeta_{j}-\eta_j\xi_j),\quad r_n:=\sum_{j=o+1}^n A_{j:n} \eta_jR_1(\Delta X_j).
\]
Let event 
\[
\mathcal{A}=\{\|\Delta X_k\|\leq b, k=o,o+1,\ldots, n\}=\{\tau_b> n\}. 
\]
Then
\begin{equation}
 \E_o \|U_n\|^2\leq \eta_{o+1:n}^2 \exp(2\lambda_H\eta_{o+1:n})(\|v_o\|^2+d\delta_0^2+\eta_o B_2). 
\label{tmp:Un}
\end{equation}
\[
\E_o \|U_n\|^4\leq 162(\eta_{o+1:n})^2 \exp(4\lambda_H\eta_{o+1:n})((\eta_{o+1:n})^2\|v_o\|^4+d^2\delta_0^4+\eta^2_o B_2^2)
\]
and 
\[
\E_o \|r_n\|^21_{\mathcal{A}} \leq  2(101\eta_{o+1:n})^3 \lambda_H^{-1}C_3^2\exp((6+12q)\eta_{o+1:n} \lambda_H) (\lambda_H^{-2} \|v_o\|^4+\delta_0^4 d^2+\eta^2_o B^2_2).
\]

\end{lem}
\begin{proof}
Since $\zeta_j$ are i.i.d. noises, it is easy to obtain an upper bound for  $U_n$ by Lemma \ref{lem:Hbounds}  a) and b),
\begin{align}
\notag
 \E_o \|U_n\|^2 &= \|(I-A_{0:n}) H^{-1} v_o\|^2+\delta_0^2 \sum_{j=o+1}^n \eta_j \text{tr}(A^2_{j:n})+ \sum_{j=o+1}^n \eta^2_j \text{tr}(A^2_{j:n} \E_o \xi_j \xi_j^T)\\
&\leq \|(I-A_{0:n}) H^{-1} v_o\|^2+d\delta_0^2 \sum_{j=o+1}^n \eta_j \|A_{j:n}\|^2+ \sum_{j=o+1}^n \eta_j\eta_o \|A_{j:n}\|^2 \text{tr}(\E_o \xi_j \xi_j^T)\\
&\leq \eta_{o+1:n}^2 \exp(2\lambda_H\eta_{o+1:n})\|v_o\|^2+(d\delta_0^2 +\eta_o B_2)\eta_{o+1:n}\exp(2\lambda_H\eta_{o+1:n}) \\
&\leq \eta_{o+1:n}^2 \exp(2\lambda_H\eta_{o+1:n})(\|v_o\|^2+d\delta_0^2+\eta_o B_2). 
\label{tmp:Un}
\end{align}
The 4th moment can be bounded first by Holder's inequality, 
\begin{align}
\notag
\E_o \|U_n\|^4 &\leq 27 \|(I-A_{0:n}) H^{-1} v_o\|^4+27\E_o \left\|\sum_{j=o+1}^n  \delta_j A_{j:n}\zeta_j\right\|^4+
27\E_o \left\|\sum_{j=o+1}^n  \eta_j A_{j:n}\xi_j\right\|^4.
\end{align}
By Lemma \ref{lem:Hbounds}  a) and b),
\[
\|(I-A_{0:n}) H^{-1} v_o\|^4\leq (\eta_{o+1:n})^4\exp(4\lambda_H\eta_{o+1:n})\|v_o\|^4. 
\]
Note that if $x_i$ are sequences of mean zero random vectors,
\[
\E \|\sum x_i \|^4= \sum_{i,j,k,l} \E \langle x_i, x_j\rangle\langle x_k, x_l\rangle \leq \sum_{i,j}\E \langle x_i, x_j\rangle^2+
\| x_i\|^2\| x_j\|^2\leq 2\sum_{i,j} \E\|x_i\|^2\|x_j\|^2
\]
Therefore
\begin{align*}
\E_o \left\|\sum_{j=o+1}^n  \delta_j A_{j:n}\zeta_j\right\|^4\leq  2\sum_{i,j=o+1}^n \delta^2_i\delta_j^2 \E_o \|A_{i:n} \zeta_i\|^2\|A_{j:n} \zeta_j\|^2&\leq 2\sum_{i,j=o+1}^n \delta^2_i\delta_j^2 \exp(4\lambda_H \eta_{o+1:n}) (d^2+2d)\\
&=6(\eta_{o+1:n})^2 \delta_o^4\exp(4\lambda_H \eta_{o+1:n})  d^2. 
\end{align*}
Likewise,
\begin{align*}
\E_o \left\|\sum_{j=o+1}^n  \eta_j A_{j:n}\xi_j\right\|^4&\leq 2 \sum_{i,j=o+1}^n \eta_i\eta_j\eta_o^2 \E_o \|A_{i:n} \xi_i\|^2\|A_{j:n} \xi_j\|^2\\
&\leq 2\eta_o^2 (\eta_{o+1:n})^2 \exp(4\lambda_H \eta_{o+1:n}) B_2^2. 
\end{align*}
In conclusion,
\[
\E_o \|U_n\|^4\leq 162(\eta_{o+1:n})^2 \exp(4\lambda_H\eta_{o+1:n})((\eta_{o+1:n})^2\|v_o\|^4+d^2\delta_0^4+\eta^2_o B_2^2)
\]
By Cauchy Schwartz, the estimates in Lemma \ref{lem:Xdev}, and finally Lemma \ref{lem:Hbounds} b), 
\begin{align}
\notag
\E_o \|r_n\|^21_{\mathcal{A}} &=  \E_o \left\|\sum_{j=o+1}^n A_{j:n}\eta_j R_1(\Delta X_j)1_{\mathcal{A}}\right\|^2\\
\notag
& \leq \left(\sum_{j=o+1}^n\eta_j\|A_{j:n}\|\right) \left(\sum_{j=o+1}^n\eta_j\|A_{j:n}\|  \E_o 1_{\mathcal{A}}\|R_1(\Delta X_j)\|^2\right)\\
\notag
&\leq  2\left(\sum_{j=1}^n\eta_j\|A_{j:n}\|\right)^2 \max_{j\leq n} C_3^2 \E_o1_{\mathcal{A}} \|\Delta X_j\|^4\\
\label{tmp:rn}
&\leq  500 q^{-7}(\eta_{o+1:n})^3 \lambda_H^{-1}C_3^2\exp((6+12q)\eta_{o+1:n} \lambda_H) (\lambda_H^{-2} \|v_o\|^4+\delta_0^4 d^2+\eta^2_o B^2_2).
\end{align}
Here we used that $\eta_o^2B_2^2\leq \eta_o B_2$. 
\end{proof}

\subsection{Proof of Lemma~\ref{lem:escape}}
With these lemmas at hand, we are ready to prove Lemma~\ref{lem:escape}.
\begin{proof}[Proof of Lemma \ref{lem:escape}]
Let $\mathcal{A}$ be the event that $\tau_b>n$, i.e. $\|X_k-X_o\|\leq b$ for all $o\leq k\leq n$.  We will decompose the desired quantity into two parts:
\begin{equation}
\label{tmp:1stdec}
\E_o F(X_{n\wedge \tau_b})-F(X_o)=\E_o (F(X_{n\wedge \tau_b})-F(X_o))1_{\mathcal{A}^c}+\E_o (F(X_{n})-F(X_o))1_{\mathcal{A}}.
\end{equation}
Let's bound the first term. We apply a 2nd order Taylor expansion of $F(x)$ near $X_o$, we find for some point $z$, the following holds
\[
|F(x)-F(X_o)|=\left| v_o^T(x-X_o)+ (x-X_o)^T\nabla^2 F(z) (x-X_o)\right|\leq C_3\|x-X_o\|^2+\|v_o\|\|x-X_o\|. 
\]
Therefore by Lemma \ref{lem:Xdev}, for a constant $C_6$,
\begin{align}
\notag
\E_o  &|F(X_{\tau_b\wedge n})-F(X_o)|1_{\mathcal{A}^c}\leq \sqrt{\E_o  |F(X_{\tau_b\wedge n})-F(X_o)|^2 \Prob_o(\mathcal{A}^c) }\\
\notag
&\leq \sqrt{2(\E_o C_3^2 \|X_{\tau_b\wedge n}-X_o\|^4+\|v_o\|^2\|X_{\tau_b\wedge n}-X_o\|^2)}  \sqrt{\Prob_o(\mathcal{A}^c) }\\
\notag
&\leq \exp((3+6q)\eta_{o+1:n}\lambda_H)\sqrt{\Prob_o(\mathcal{A}^c)}\\
\notag
&\quad \quad\cdot\bigg(17q^{-1.5}(\eta_{o+1:n})^{1.5} C_3\lambda_H^{-1/2} (\lambda_H^{-1}\|v_o\|^2+ \delta_0^2 d+\eta_o B_2)\\
&\quad\quad\quad+3\|v_o\|\sqrt{q^{-1}\eta_{o+1:n}} \lambda_H^{-1/4}\sqrt{\lambda_H^{-1} \|v_o\|^2+ \delta_o^2 d+\eta_o B_2}\bigg)\\
\label{tmp:fAc}
&\leq \eta_{o+1:n}C_6\exp((3+6q)\eta_{o+1:n}\lambda_H)\sqrt{\Prob_o(\mathcal{A}^c)}\delta_0^2 d.
\end{align}
In the last step above, we use the following comes from  by our parameter setting, 
\begin{equation}
\label{tmp:alltodelta}
\eta_oB_2\leq \delta_0^2,\quad \|v_o\|\leq \delta_0\sqrt{\min\{d\lambda_H, d, d\eta_{o+1:n}^{-1}\}}. 
\end{equation}

Next, we bound $\E_o (F(X_{n})-F(X_o))1_{\mathcal{A}}$.

For the $\tau_b\geq n$ case,  we employ the Taylor expansion of $F(x)$ near $X_o$:
\begin{align*}
F(X_{n})-F(X_o)&=v_o^T\Delta X_{n}+ \Delta X_{n}^TH \Delta X_{n}+R_0(\Delta X_{n})\\
&=(v_o^T U_n+U_n^T H U_n)+ v_o^T r_n +2r_n^T H U_n+r_n^T H r_n +R_0(\Delta X_n).
\end{align*}
This leads to the following bound
\begin{align}
\notag
\E_o (F(X_{n})-F(X_o) )1_{\mathcal{A}}\leq &\E_o(v_o^T U_n+U_n^T H U_n)+\E_o \left(v_o^T r_n +2r_n^T H U_n+r_n^T H r_n +R_0(\Delta X_n)\right) 1_{\mathcal{A}}\\
&+\E_o|v_o^T U_n+U_n^T H U_n|1_{\mathcal{A}^c}\label{tmp:2nddec}
\end{align}

Note that $A_{j-1:n}+\eta_jH A_{j:n}=A_{j:n}$, so 
\[
\E v_o^T U_n =-\sum_{j=1}^n \eta_j v_o^T A_{j:n} v_o=-v_o^T  (I-A_{0:n})H^{-1}v_o. 
\]
And by independence, we obtain $\E_o U_n^T H U_n$
\begin{align*}
\E_o U_n^T H U_n &=v_o^T\left(\sum_{j=o+1}^n A_{j:n} \eta_j\right) H\left(\sum_{j=o+1}^n A_{j:n} \eta_j\right)  v_o+\sum_{j=o+1}^n \delta_j^2 \text{tr} A_{j:n} H A_{j:n}+\sum_{j=o+1}^n \eta_j^2 \text{tr} A_{j:n} H A_{j:n} \E_o \xi_j \xi_j^T \\
&=v_o^T (I-A_{o:n})H^{-1} (I-A_{o:n}) v_o+\delta_0^2 \sum_{j=o+1}^n \eta_j \text{tr} A_{j:n} H A_{j:n}+\sum_{j=o+1}^n \eta_j^2 \text{tr} A_{j:n} H A_{j:n} \E_o \xi_j \xi_j^T .
\end{align*}
So by Lemma \ref{lem:Hbounds} c) and e), and that $\text{tr}(CB)\leq \lambda_{max}(C)\text{tr}(B)$ for all symmetric $C$ and PSD matrix $B$, we find that 
\begin{align*}
\E_o v_o^T U_n+U_n^T H U_n&=-v_o^T (I-A_{0:n})H^{-1} A_{0:n} v_o+\delta_0^2 \sum_{j=o+1}^n \eta_{j} \text{tr} A_{j:n} H A_{j:n}+\sum_{j=o+1}^n \eta_j^2 \text{tr} A_{j:n} H A_{j:n} \E_o \xi_j \xi_j^T\\
&\leq \sum_{j=o+1}^n \delta_0^2 \eta_{j} \text{tr} A_{j:n} H A_{j:n}+ B_2\eta_j^2 \lambda_{\max} (A_{j:n} H A_{j:n}) \\
&\leq -4\delta_0^2 \eta_{o+1:n}.
\end{align*}
Plug in estimates from Lemma \ref{lem:Urbound}, note that $\eta_{o+1:n}^2\leq (\eta_{o+1:n})^2$, 
\begin{align}
\notag
\E_o (v_o^T U_n+U_n^T H U_n)1_{\mathcal{A}^c}&\leq\|v_o\| \sqrt{\Prob_o(\mathcal{A}^c)\E_o \|U_n\|^2}+\|H\|\sqrt{\Prob_o(\mathcal{A}^c)\E_o \|U_n\|^4}\\
\notag
&\leq \eta_{o+1:n}\exp(\lambda_H\eta_{o+1:n})\sqrt{\Prob_o(\mathcal{A}^c)}d\delta_0^2+8C_3\eta_{o+1:n} \exp(2\lambda_H\eta_{o+1:n})d\delta_0^2\sqrt{\Prob_o(\mathcal{A}^c)}\\
&\leq \eta_{o+1:n}C_6 \exp(2\lambda_H\eta_{o+1:n})d\delta_0^2\sqrt{\Prob_o(\mathcal{A}^c)}. 
\label{tmp:voun}
\end{align}
Recall that under our conditions,
\[
\max\{\|v_o\|^2, \eta_{o+1:n}\|v_o\|^2, \lambda_H^{-1}\|v_o\|^2,\eta_o B_2\}\leq \delta_0^2 d. 
\]
So by Young's inequality, and Lemma \ref{lem:Urbound},  we can increase $C_6$  so that the following hold
\begin{align*}
\E_o v_o^T r_n 1_\mathcal{A}&\leq \|v_o\|\sqrt{ \E_o \|r_n\|^21_\mathcal{A}}\leq C_6 (\eta_{o+1:n})^{1.5} \exp((3+6q)\eta_{o+1:n} \lambda_H) \delta_0^3 d^{3/2}
\end{align*}
\[
2\E_o1_\mathcal{A} r_n^T H U_n\leq 2 \sqrt{\E_o 1_\mathcal{A}\|r_n\|^2 \E_o 1_\mathcal{A}\|HU_n\|^2}\leq \lambda_H^{-0.5}
C^2_6 (\eta_{o+1:n})^2\exp((6+12q)\eta_{o+1:n} \lambda_H) \delta_0^3 d^\frac32,\\
\]
\[
\E_o 1_\mathcal{A} r_n^T H r_n\leq \E_o 1_\mathcal{A} C_3\|r_n\|^2\leq C^3_6 \lambda_H^{-1}  (\eta_{o+1:n} )^3\exp((6+12q)\eta_{o+1:n} \lambda_H)\delta_0^4 d^2.
\]
Because of Young's inequality and Lemma \ref{lem:Xdev}
\[
\E_o 1_\mathcal{A}  R_0(\Delta X_n)\leq C_3 \E_o1_\mathcal{A}  \|\Delta X_n\|^3\leq 
C_3 [\E_o1_\mathcal{A}  \|\Delta X_n\|^4]^\frac{3}{4}\leq
 C_6\lambda_H^{-\tfrac{3}{4}}  (\eta_{o+1:n})^3 \exp((9+18q)\eta_{o+1:n} \lambda_H)\delta_0^3 d^\frac32. 
\]
Under our assumptions, $\delta_0 \sqrt{d}\leq 1$ so 
\begin{equation}
\label{tmp:3rdorder}
\E_o1_\mathcal{A} (v_o^T r_n +2r_n^T H U_n+r_n^T H r_n+ R_0(\Delta X_n))\leq \eta_{o+1:n}\exp((9+18q)\eta_{o+1:n} \lambda_H) (C_6^3\delta_0^3 d^\frac32+C_6^5\delta_0^4 d^2).
\end{equation}

Put these estimates back in \eqref{tmp:2nddec}, and eventually in \eqref{tmp:1stdec}, we find
\begin{align}
\E_o F(X_{n\wedge \tau_b})-F(X_o)&\leq -4\delta_0^2 \eta_{o+1:n} \tag{$\E_o v_o^T U_n+U_n^T H U_n$}\\
&\quad +\eta_{o+1:n}\exp((9+18q)\eta_{o+1:n} \lambda_H) (C_6^3\delta_0^3 d^\frac32+C_6^5\delta_0^4 d^2)\tag{from \eqref{tmp:3rdorder}}\\
&\quad+\eta_{o+1:n}C^2_6 \exp(2\lambda_H\eta_{o+1:n})\sqrt{\Prob_o(\mathcal{A}^c)}d\delta_0^2\tag{from \eqref{tmp:voun}}\\
&\quad+ \eta_{o+1:n}C^2_6\exp((3+6q)\eta_{o+1:n}\lambda_H)\sqrt{\Prob_o(\mathcal{A}^c)}\delta_0^2 d\tag{from \eqref{tmp:fAc}}
\end{align}
We also need to plug in the bound for $\sqrt{\Prob_o(\mathcal{A}^c)}$ from Lemma \ref{lem:Xdev}, which is simplified by \eqref{tmp:alltodelta}:
\[
\sqrt{\Prob_o(\tau_b\leq n)}\leq C_6^{5}\exp((2+6q)\eta_{o+1:n} \lambda_H) \delta_0^2d.
\]
In conclusion, we have
\begin{align*}
\frac{\E_o F(X_{n\wedge \tau_b})-F(X_o)}{\eta_{o+1:n}}&\leq -4\delta_0^2+C_6^3\exp((9+18q)\eta_{o+1:n} \lambda_H) \delta_0^3 d^\frac32+C^7_6\exp((5+12q)\eta_{o+1:n} \lambda_H)\delta_0^4 d^2. 
\end{align*}
Under our conditions, 
\[
 1\geq C_6^3\exp((9+18q)\eta_{o+1:n} \lambda_H)\delta_0  d^\frac32,\quad 1\geq C_6^7\exp((5+12q)\eta_{o+1:n} \lambda_H)\delta^2_0  d^2,
\]
This leads to our final claim. 
\end{proof}

\subsection{Proof of Theorem~\ref{thm:second}}
To conclude this section, we give the proof of Theorem~\ref{thm:second} below.
\begin{proof}[Proof of Theorem~\ref{thm:second}]
Consider the following sequence of stopping times with $\tau_0=0$,
\[
\tau_{k+1}=\begin{cases}
\tau_k &\text{if }\tau_{sosp}\leq \tau_k;\\
\tau_k+1 &\text{if }\|\nabla F(X_{\tau_k})\|\geq 2\epsilon_0,\tau_{sosp}>\tau_k;\\
\min\{t: \eta_{\tau_k+1:t}\geq 2D_k,\text{or}\,\, \|X_t-X_{\tau_k}\|\geq C_3/2\lambda_H\} &\text{otherwise}. 
\end{cases}
\]
Here $D_k=\frac{2}{\lambda_H}\log \frac{16D_4+40}{\lambda_H}$ with $\lambda_H=\lambda_{\max}(-\nabla^2 F(X_{\tau_k}))$. 

 If $\tau_{sosp}\leq \tau_k$, then the following holds trivially
\[
\E_{\tau_k} F(X_{\tau_{k+1}})= F(X_{\tau_k})-\epsilon_0^2 \eta_{\tau_{k}+1:\tau_{k+1}},\quad a.s..
\]
By \eqref{tmp:1step} and strong Markov property, if $\|\nabla F(X_{\tau_k})\|\geq 2\epsilon_0$, 
\[
\E_{\tau_k} F(X_{\tau_{k+1}})\leq F(X_{\tau_k})-\epsilon_0^2 \eta_{\tau_{k}+1}=F(X_{\tau_k})-\epsilon_0^2 \eta_{\tau_{k}+1:\tau_{k+1}},\quad a.s..
\]
If $\|\nabla F(X_{\tau_k})\|\leq 2\epsilon_0$, $\tau_{sosp}>\tau_k$, then $\lambda_{\max}(-\nabla^2F(X_{\tau_k}))>\lambda_\epsilon$.
By Lemma \ref{lem:escape} and strong Markov inequality, we also obtain the same inequality that
\[
\E_{\tau_k} F(X_{\tau_{k+1}})\leq F(X_{\tau_k})-\epsilon_0^2 \eta_{\tau_{k}+1:\tau_{k+1}},\quad a.s..
\]
Next we add in the requirement that time horizon is before $N$. 
Note that 
 \[
\unit_{\tau_{k+1}\geq N}F(X_{\tau_{k+1}\wedge N})- \unit_{\tau_{k}\leq N}F(X_{\tau_k\wedge N})\leq  \unit_{\tau_{k}\leq N\leq \tau_{k+1}} C_0,\quad \epsilon_0^2 Q\geq \epsilon_0^2\eta_{\tau_k+1:\tau_{k+1}\wedge N},\quad a.s..
 \]
 we have
 \[
 \E_{\tau_k} F(X_{\tau_{k+1}\wedge N})- F(X_{\tau_k\wedge N})\leq   \E_{\tau_k}\unit_{\tau_{k+1}\geq N}(C_0 +\epsilon_0^2 Q)-\epsilon_0^2\eta_{\tau_k+1:\tau_{k+1}\wedge N},\quad a.s..
 \]
Also from our previous derivations, 
\begin{align*}
&\E_{\tau_k} \unit_{\tau_{k+1}\leq N} F(X_{\tau_{k+1}})- \unit_{\tau_{k}\leq N}  F(X_{\tau_k})\\
&=\unit_{\tau_{k}\leq N}\E_{\tau_k}( F(X_{\tau_{k+1}})-  F(X_{\tau_k}))-\unit_{\tau_{k}\leq N\leq \tau_{k+1}}  F(X_{\tau_{k+1}})\\
&\leq -\unit_{\tau_{k}\leq N}\epsilon_0^2\E_{\tau_k} \eta_{\tau_{k}+1:\tau_{k+1}}-\E_{\tau_k}\unit_{\tau_{k}\leq N\leq \tau_{k+1}}  F(X_{\tau_{k+1}}),\quad a.s..
\end{align*}
Let $K$ be the first $k$ such that $\tau_k\geq N$ or $\tau_k\geq \tau_{sosp}$. Then $\tau_{sosp}\wedge N\leq \tau_K\wedge N$. Summing the previous inequalities for all $k\leq K$, and take total expectation,
\[
\E \unit_{\tau_K\leq N}F(X_{\tau_{K}\wedge N})-F(X_0)\leq  \E -\epsilon_0^2 \eta_{1:\tau_{K-1}}-\unit_{N\leq \tau_K}F(X_{\tau_K}).
\]
So we find that 
\[
\E \eta_{1:\tau_{K-1}}\leq \frac{1}{\epsilon_0^2} \E (F(X_{\tau_K\wedge N})-F(X_0))\leq \frac{2C_0}{\epsilon_0^2}. 
\]
Then note that for any $k$ $\eta_{\tau_{k-1}:k}\leq 2Q$, so 
\[
\E \eta_{1:\tau_{K}}\leq 2Q+\frac{2C_0}{\epsilon_0^2}. 
\]
So by Markov inequality 
\[
\Prob(\tau_{sosp}>N )\leq \frac{\E \eta_{1:\tau_K}}{\eta_{1:N}}\leq \frac{2Q+2C_0\epsilon_0^{-2}}{\eta_{1:N}}. 
\]

\end{proof}

	\section{Technical Proofs of ergodicity }
\label{sec:proof_ergodic}

\subsection{Time points and recurrence}
Since our SGLD is a time-inhomogeneous Markov chain, the first step is to design a sequence of time points $n_i$ so that $\{X_{n_i},i\geq 0\}$  is approximately a time-homogeneuous Markov chain. This is done by Lemma \ref{lem:stoppingtimes}, where $\delta$ will be chosen as $\epsilon_0$ when applied to Theorem \ref{thm:ergodic}. 

Once we have established these time points, we wish to show $X_{n_i}$ visits a compact set infinite many times. The compact set is chosen as a sub-level set of $F$. 

\begin{lem}
\label{lem:stoppingtimes}
Under Assumption \ref{aspt:coercive}, we have the following iteration indices:
\begin{enumerate}[a)]
\item Let $n_0$ be an index such that $\eta_n\leq \delta$ when $n\geq n_0$.  Let $n_k$ be sequence of iteration index $n_{k+1}=\inf\{s>n_{k}+1: \eta_{n_k+1:s}\geq \delta\}$, then $\eta_{n_{k}+1:n_{k+1}}\leq 2\delta$. 
\item We have the following Lyapunov-type inequality
\begin{equation}
\label{eqn:Lyapunov}
\E_{n_{k}} F(X_{n_{k+1}})\leq  \exp(-\tfrac {c_7}2\delta ) F(X_{n_k})+\delta \left( D_7+6\delta B_1+6d\delta_{0}^2\right). 
\end{equation}
\item If we let \[
K:=\left\lceil\frac{2\log \frac12 M_V /\E F(X_{n_0})}{c_7 \delta} \right\rceil,
\]
then
\[
\E F(X_{n_K})\leq M_V:=\frac{8}{c_7}\left( D_7+6\delta B_1+6d\delta_{0}^2\right).
\]
\item Define a sequence of stopping times with $\tau_0=K$,
\[ 
 \tau_{k+1}=\inf\{t: t\geq \tau_k +1, F(X_{n_{t}})\leq M_V\}, \quad k \ge 1;
\]
Then 
\[
\E \tau_j\leq K+j+\frac{8j}{c_7\delta}, \quad j\ge 1.
\]
\end{enumerate}
\end{lem}
\begin{rem}
If we have further tail conditions of the stochastic perturbation, it is often possible to show $e^{\lambda F}$ also have Lyapunov-type inequality. This will provide an exponential tail bound for $\tau_j$, and lead to geometric convergence in Theorem \ref{thm:ergodic}. 
\end{rem}

\begin{proof}
For claim a), simply note the $\eta_n$ decrease to zero as $n\to \infty$. Moreover, by the definition of $n_k$
\[
\eta_{n_{k}+1:n_{k+1}}\leq \eta_{n_{k}+1:n_{k+1}-1}+\eta_{n_{k+1}}\leq 2\delta. 
\]
For claim b), recall the proof of Theorem 3.3, we have 
\[
\E_n F(X_{n+1})\leq F(X_n)-\frac12\eta_{n+1}\|\nabla F(X_n)\|^2+3\eta_{n+1}^2 B_1+3d\delta_{0}^2\eta_{n+1}.
\]
We require that $\|\nabla F(X)\|^2\geq c_7 F(X)-D_7$, so 
\begin{align*}
\E_n F(X_{n+1})&\leq (1-\tfrac {c_7} 2 \eta_{n+1}) F(X_n)+ \frac12 \eta_{n+1}D_7+3\eta_{n+1}^2 B_1+3d\delta_{0}^2\eta_{n+1}\\
&\leq \exp(-\tfrac {c_7} 2 \eta_{n+1}) F(X_n)+ \frac12 \eta_{n+1}D_7+3\eta_{n+1}^2 B_1+3d\delta_{0}^2\eta_{n+1}.
\end{align*}
By iterating this inequality $m$ times, we have
\begin{align*}
\E_n F(X_{n+m})&\leq \exp(-\tfrac {c_7}2 \eta_{n+1:n+m}) F(X_n)+\sum_{k=n+1}^{n+m}\frac12 \eta_{k}D_7+3\eta_{k}^2 B_1+3d\delta_{0}^2\eta_{k}\\
&\leq \exp(-\tfrac {c_7}2 \eta_{n+1:n+m}) F(X_n)+\eta_{n+1:n+m}\left(\frac12 D_7+3\eta_{n+1} B_1+3d\delta_{0}^2\right). 
\end{align*}
In particular, we have 
\[
\E_{n_k} F(X_{n_{k+1}})\leq \exp(-\tfrac {c_7}2\delta ) F(X_{n_k})+\delta\left( D_7+6\delta B_1+6d\delta_{0}^2\right). 
\]
For claim c), iterate \eqref{eqn:Lyapunov} $k$ times, we have for small enough $\delta$
\[
\E F(X_{n_k})\leq \exp(-\tfrac{c_7k}{2}\delta )\E F(X_{n_0})+\frac{4}{c_7}\left( D_7+6\delta B_1+6d\delta_{0}^2\right),
\]
and we can find
\[
K:=\left\lceil\frac{2\log \frac12 M_V /\E F(X_{n_0})}{c_7 \delta} \right\rceil,
\]
such that $\E F(X_{n_{K}})\leq M_V$. So if we increase $n_0$ to $n_K$ we have the claimed result.

For claim d), consider stopping time $\tau:=\inf\{k\geq K: F(X_{n_k}) \leq  M_V \}$, denote $V(k)=F(X_{n_k})$. Then from claim b) we have 
\[
\E_{n_k} V(k+1)\leq (1-\gamma)V(k)+B_V,\quad B_V:=\delta \left( D_7+6\delta B_1+6d\delta_{0}^2\right), \quad \gamma:=1-\exp(-\tfrac {c_7}{2}\delta). 
\]
We will verify that for any $t$
\begin{equation}
\label{eqn:supermart}
\E_{n_t}V(\tau \wedge (t+1))+ 2B_V\tau\wedge (t+1) \leq V(\tau\wedge t)+2B_V \tau\wedge t\quad a.s..
\end{equation}
To see this, note that if $\tau\leq t$, the inequality \eqref{eqn:supermart} trivially holds. And if $\tau\geq t+1$, 
\[
V(t)\leq M_V\leq \frac{4 B_V }{\gamma},
\]
 and it suffices to show
\[
\E_{n_t} V(t+1)+2B_V \leq V(t)\quad a.s..
\]
But this can be obtained by observing that $-\gamma  V(t)\geq -4B_V$,  so
\[
\E_{n_t} V(t+1)\leq V(t)- \gamma V(t) +2B_V\leq V(t)-2B_V\quad a.s..
\]
With \eqref{eqn:supermart} verified, we know that  $V(\tau\wedge t)+2B_V \tau\wedge t$ is a supermartingale, therefore by letting $t\to\infty$, we find that
\[
2B_V\E_{n_K}\tau \leq \E_{n_K} V(\tau)+2B_V\E_{n_K} \tau\leq V(K)+2B_VK\quad a.s..
\]
In particular, note that $\tau_{k+1}$ is essentially the next $\tau$ after $\tau_k+1$, so by strong Markov property,
\[
2B_V\E_{n_{\tau_k}} (\tau_{k+1}-\tau_k-1) \leq V(\tau_k+1) \quad a.s..
\]
Taking total expectation, we have
\[
2B_V\E \tau_{k+1}-\tau_k-1\leq \E V(\tau_k+1).
\]
Adding over all $k$, note that $\tau_0=K$, we have
\[
\E \tau_{j}\leq K+j+\frac1{B_V}\sum_{k=0}^j\E V(\tau_k+1).
\]
Also note that $V(\tau_k)\leq M_V$, so
\[
\E V(\tau_{k}+1)=\E \E_{n_{\tau_k}}V(\tau_{k+1})\leq \left(1-\gamma\right) \E V(\tau_k)+B_V\leq (1-\gamma)M_V+B_V\leq M_V.
\]
So $\E \tau_j\leq j+K+j\frac{M_V}{B_V}$.
\end{proof}
\subsection{Reachablity}
In Markov chain analysis, one crucial step showing ergodicity is verifying that when the Markov chain starts from a recurrent compact set, it has a positive chance of reach a target. In our context, we wish to show the SGLD can visit a target point $z_0$ with positive probability bounded from below. 
\begin{lem}
\label{lem:small}
Under Assumption \ref{aspt:coercive}, suppose $\|X_{n_i}\|\leq D_X, \|z_0\|\leq D_X$, then for any $\epsilon>0$
\[
\Prob\left(\|X_{n_{i+1}}-z_0\|\leq (8D_X+1)\epsilon+2\delta\sqrt{B_4} \right)>c_\alpha:=\frac14\Prob\left(\|Z-\tfrac{2D_X }{\delta_0\sqrt{\delta}}e_1\|\leq \tfrac{\epsilon}{\delta_0\sqrt{2\delta}}\right), 
\]
where $Z\sim N(0,I_d)$ and $e_1=[1,0,0,\ldots,0]$ is the first Euclidean basis vector. 
\end{lem}

\begin{proof}
For notational simplicity, we let $n_i=o$ and $m=n_{i+1}-n_i$.  By iterating \eqref{sys:SGLD} $k$ times, we write 
\begin{align*}
X_{o+k}=X_{o}+\sum_{j=1}^k\eta_{o+j} \nabla F(X_{o+j-1})+Y_k+Z_k
\end{align*}
where 
\[
Y_k:=\sum_{j=1}^{k}\eta_{o+j} \xi_{o+j},\quad Z_k:=\sum_{j=1}^{k}\delta_0 \sqrt{\eta_{o+j}}\zeta_{o+j}.
\]
Let $z=z_0-X_o$. Denote event:
\[
\calA:=\{\|Z_{m}-z\|\leq \epsilon, \|Z_{k}\|\leq \|z\|+\epsilon+2,\|Y_k\|\leq 2\delta\sqrt{B_4}, k=1,\ldots,m\}.
\] 
Apply Lemma \ref{lem:Brownian} with $a_j=\eta_{o+j}$, we know that 
\[
\Prob(\calA)\geq \frac14  \Prob\left(\|Z-\tfrac{z}{\delta_0\sqrt{\delta}}\|\leq \delta_0\sqrt{\delta/2}\right)
\]
Note that the lower bound decreases as $\|z\|$ increase, so we can obtain an lower bound by considering $z$ with the maximum norm. 
Now consider when $\calA$ takes place. Note that if $\|X_{o+j}\|\leq 4D_X$ for all $j\leq k-1$ then $\nabla F(X_{o+j})\leq D_F$, 
\[
\|X_{o+k}\|\leq \|X_{o}\|+\eta_{o+1:o+k} D_F+\|z\|+\epsilon+2\delta\sqrt{B_4}\leq 4D_X,
\]
so $\|X_{o+j}\|\leq 4D_X$ for all $j=1,\ldots, m$. Moreover, 
\[
\|X_{o+m}-z_0\|\leq \|Z_m-z_0\|+\|Y_m\|+\eta_{o+1:o+m}D_F\leq (2D_F+1)\epsilon+2\delta\sqrt{B_4}. 
\]
\end{proof}

\begin{lem}
\label{lem:Brownian}
For any sequence $a_k>0$ such that $\delta \leq a_{1:n}:=\sum_{j=1}^n a_j\leq 2\delta$ and  $\delta<\frac1d$, the following holds:
\begin{enumerate}[a)]
\item Suppose we let
\[
Z_{k}=\sum_{j=1}^{k}\delta_0\sqrt{a_j}\zeta_j,\quad \zeta_j\sim \mathcal{N}(0, I_d),\quad k\geq o+1
\]
Then for any target vector $z$ and distance $r$
\[
\Prob(\|Z_n-z\|\leq r, \|Z_k\|\leq \|z\|+r+2\delta_0,k=1,\ldots, n)\geq \frac34\Prob(\|Z-\tfrac{z}{\delta_0\sqrt{\delta}}\|\leq \tfrac{r}{\delta_0\sqrt{2\delta}})
\]
where $Z$ is a random variable follows $\mathcal{N}(0, I_d)$. 
\item Let $\mathcal{F}_Z$ denote the $\sigma$-algebra generated by $Z_1,\ldots, Z_n$ as in a). Suppose $\xi_k$ is a sequence of random vectors, such that 
\[
\E (\xi_k |\mathcal{F}_Z,\xi_j,j<k)=\mathbf{0},\quad \E  (\|\xi_k\|^2 |\mathcal{F}_Z,\xi_j,j<k)\leq B_4.
\]
Let $Y_k=\sum_{j=1}^k a_j\xi_j$  then 
\[
\Prob(\|Y_k\|\leq 2\delta\sqrt{ B_4}|\mathcal{F}_Z)\geq \frac12. 
\]
\end{enumerate}

\end{lem}
\begin{proof}
For claim a), it is easy to verify the joint distribution of $[Z_{1},\ldots, Z_n]/\delta_0$ is the same as the distribution of 
$[W_{t_1},\ldots, W_{t_n}]$, where $t_i=a_{1:i}$ and $W$ is a $d$-dimensional Wiener process.
Note that 
\begin{align*}
\Prob(\|W_{t_n}-z\|\leq  r/\delta_0)=\Prob(\|\sqrt{a_{1:n}}Z-z\|\leq r/\delta_0)&=\Prob(\|Z-z/\sqrt{a_{1:n}}\|\leq r/\delta_0\sqrt{a_{1:n}})\\
&\leq \Prob(\|Z-z/\sqrt{\delta}\|\leq r/\delta_0\sqrt{2\delta}).
\end{align*}
And when conditioned on $W_{t_n}=w$, $W_{t<t_n}$ is known as a Brownian bridge. By Karatzas and Shreve, The distribution of its path is the same as 
\[
X_t=\frac{t}{t_n}w+(B_t-\frac{t}{t_n}B_{t_n}),
\]
where $B_t$ is another independent Wiener process. So with any $w, \|w\|\leq \|z\|+r$,
\[
\Prob(\|W_t\|\leq \|z\|+r+2,\forall t\leq t_n|W_{t_n}=w)=\Prob(\|X_t\|\leq \|z\|+r+2,t\leq t_n)\geq \Prob(\|B_t\|\leq 2,\forall t\leq t_n). 
\]
Consider $b_t=\|B_t\|^2-td$. It is easy to check that $b_t$ is a martingale. Let $\tau=\inf\{t: \|B_t\|^2\geq 4d\}$. Then 
\[
\E \|B_{t\wedge\tau}\|^2=\E b_{t\wedge \tau}+d\E t\wedge \tau\leq dt_n. 
\]
So by Markov inequality
\[
\Prob(\|B_{t}\|^2\geq 4\text{ for some }t\leq t_n)\leq \Prob(\|B_{t\wedge \tau}\|^2\geq 4)\leq \frac{dt_n}{4}\leq \frac14. 
\]
Claim a) can be obtained by tower property. 

For claim b), note that $y_k=\|Y_k\|^2-a_{1:k}B_4 $ is a supermartingale conditioned on $\mathcal{F}_Z$. Let $\tau$ be the first time that $\|Y_{k}\|^2\geq 4 \delta B_4$, then the expectation conditioned on $\mathcal{F}_Z$ yields
\[
\E_Z \|Y_{\tau\wedge k}\|^2\leq \E_Z a_{1:\tau\wedge k}B_4+\E y_{k\wedge \tau}\leq  B_4\E_Z a_{1:\tau\wedge k}\leq 2B_4\delta
\]
So Markov inequality leads to 
\[
\Prob_Z(Y_{k}\leq 2\sqrt{\delta B_4},\forall k=1,\ldots, n)\geq \frac12. 
\]

\end{proof}

\subsection{Proof for Theorem \ref{thm:ergodic}}
Once we have both recurrence and reachability, it is intuitive to see why SGLD can visit an arbitrary point $z_0$ in the space: SGLD will visit a sub-level set of $F$ infinitely many times, and each time it has a positive probability to visit  $z_0$.
\begin{proof}
Let $\delta$ satisfies the following:
Let $n_0$ be an index such that $\eta_n\leq \delta$ when $n\geq n_0$.  Let $n_k$ be sequence of iteration index $n_{k+1}=\inf\{s>n_{k}+1: \eta_{n_k+1:s}\geq \delta\}$. 

Define a sequence of stopping times with $\tau_0=K$,
\[ \tau_0=t_0, \quad
 \tau_{k+1}=\inf\{t: t\geq \tau_k +1, F(X_{n_{t}})\leq M_V\}, \quad k \ge 1;
\]
and  the stopping time $$\tau_*=\inf\{t>0: \|X_{n_t}-z_0\|\leq\epsilon \}.$$
To prove Theorem \ref{thm:ergodic}, it suffices  to show that $\Prob(\tau_*\geq T)\leq p_0$ with a choice of $T$. Then note that for any $J$
\begin{align*}
\Prob(\tau_*\geq T)&=\Prob(\tau_*\geq T,\tau_J > T)+\Prob(\tau_*\geq  T,\tau_J \leq T)\\
&\leq \Prob(\tau_J>T)+\Prob(\|X_{n_{\tau_k+1}}-z_0\|>\epsilon ,\tau_J\leq T, k=1\cdots, J)\\
&\leq \Prob(\tau_J>T)+\Prob(\|X_{n_{\tau_k+1}}-z_0\|>\epsilon, k=1\cdots, J).
\end{align*}
Lemma \ref{lem:stoppingtimes} shows that 
\[
\E \tau_J\leq K+J+\frac{8J}{c_7\delta}.
\] 
In particular, by Markov inequality, we can choose an $T\geq \frac{K+J+\frac{8J}{c_7\delta}}{2p_0}$ such that
\[
\Prob(\tau_J\geq T)\leq \frac{\E\tau_K}{T}\leq  \frac12p_0.
\]
Note that $F(X_{n_{\tau_k}})\leq M_V$ implies $\|X_{n_{\tau_k}}\|\leq D_X$.
So by Lemma \ref{lem:small}, there is a $c_\alpha$
such that 
\[
\Prob_{n_{\tau_k}}(\|X_{n_{\tau_k+1}}-z_0\|>(8D_X+1)\epsilon+2\delta\sqrt{B_4})\leq 1-c_\alpha.
\] 
Because that $V_i(\tau_k)\leq M_V$, so
\[
\Prob(\|X_{n_{\tau_k+1}}-z_0\|>\epsilon, k=1\cdots, J)=\E \prod_{k=1}^J \Prob_{n_{\tau_k}} (\|X_{n_{\tau_k+1}}-z_0\|>\epsilon)\leq (1-c_\alpha)^J.
\]
Pick a large $J=\lceil \frac{\log \frac12p_0}{\log(1-c_\alpha)}\rceil$ so that $(1-c_\alpha)^J\leq \frac12 p_0$. In summary we need
\[
T\leq  \frac{\left\lceil\tfrac{2\log \frac12 M_V /\E F(X_{n_0})}{c_7 \delta} \right\rceil+\lceil\frac{\log \frac12p_0}{\log(1-c_\alpha)}\rceil (1+\tfrac{8}{c_7\delta})}{2p_0}.
\]
Also recall that $2\delta T\geq \eta_{n_0:n_T}\propto \eta_0  n_T^{1-\alpha}$, so $n_T\propto [\frac{T}{2\delta\eta_0}]^{\frac{1}{1-\alpha}}$. 
\end{proof}

\subsection{\blue{Proof of Lemma \ref{lem:x2}}}
\begin{proof}
It is easy to see that $X_n$ and $Y_n$ are both Gaussian distributed with mean being zero, given the linearity. The variance at step $n$ follow the update
\[
V^x_n=(1-\eta_n)^2 V^x_{n-1}+\eta_n,\quad V^y_n=(1-\eta_n)^2 V^y_{n-1}+\eta^2_{n}. 
\]
For any fixed $\epsilon>0$, there is an $m$ so that  $\eta_n\leq \epsilon$ for all $n\geq m$. 

Then if $V^y_{n-1}\geq 2\epsilon$
\[
V^y_n\leq (1-\eta_n)V^y_{n-1}-\eta_n(V^y_{n-1}-\eta_nV^y_{n-1}-\eta_n)\leq (1-\eta_n)V^y_{n-1}\leq V^y_{n-1}.
\]
If $V^y_k\geq 2\epsilon$ for $m\leq k\leq n, V^y_n\leq \exp(-\eta_{m+1:n})V^y_m$. Since $\eta_{m+1:n}\to \infty$,  there is a $\tau$ such that $V^y_\tau\leq 2\epsilon$. And after $n\geq \tau$, there is two scenarios, if $V^y_{n-1}\geq 2\epsilon$, $V^y_n\leq V^y_{n-1}$. If $V^y_{n-1}\leq 2\epsilon, V^y_n\leq V^y_{n-1}+\eta^2_n\leq 2\epsilon+4\epsilon^2.$ So we conclude that
\[
\lim\sup_{n\to \infty}V^y_n\leq 2\epsilon+4\epsilon^2.
\]
Since this holds for all $\epsilon$, we find that $\lim_{n}V^y_n=0$. As for $V^x_n$, we set $\Delta^x_n=V^x_n-\frac12$. We find its update rule follows
\[
\Delta^x_n=(1-\eta_n)^2\Delta_n^x+\frac12\eta_n^2.
\]
This is similar to the update rule of $V^y_n$. So by the same arguments, we find that $\Delta^x_n\to 0$ and $V^x_n\to \frac12$. 
\end{proof}

	\section{Technical Proofs in Section~\ref{sec:examples}}
	\label{sec:proof_examples}
In this section, we provide detailed calculations of the assumption constants for each example discussed in Section \ref{sec:examples}

\begin{proof}[Proof of Proposition~\ref{prop:lr}]
	The Hessian of $F$ is given by $A$, which is PSD, so $D_4, C_2$ can be bounded by  tr$(A)$.  
	Next, we check the noise term in the stochastic gradient: 
	\[
	\xi_{n+1}=A(X_{n+1}-x^*)-a_{n+1}(a_{n+1}^TX_n-b_{n+1})=(A-a_{n+1}a_{n+1}^T)(X_n-x^*)+ \epsilon_{n+1} a_{n+1}. 
	\]
	By H\"{o}lder's inequality,  there is a constant $c$ such that 
	\begin{align*}
	\E_n \|\xi_{n+1}\|^4&\leq 8\E_n  \|(A-a_{n+1}a_{n+1}^T)(X_n-x^*)\|^4+8 \E_n \|\epsilon_{n+1}a_{n+1}\|^4 \\
	&\leq 8 \|X_n-x^*\|^4 \E_n (\|A\|+\|a_{n+1}\|^2 )^4+ 24 \E_n \|a_{n+1}\|^4\\
	&\leq 64 \|X_n-x^*\|^4\|A\|^4+ 64\times 105\|X_n-x^*\|^4\tr(A)^4+24\times 3 \text{tr}(A)^2\\
	&\leq c^2\tr(A)^4(\|X_n-x^*\|^4+1). 
	\end{align*}
	Therefore we can choose $B_2=c\tr(A)^2(\Gamma^2+1)$. By Cauchy's inequality, then 
	\[
	\E_n \xi_{n+1}^T \nabla^2 F(X_n)\xi_{n+1}\leq \tr(A) \E_n \|\xi_{n+1}\|^2\leq c \tr(A)^3(\Gamma^2+1). 
	\]
	
	To check Assumption \ref{aspt:coercive}, note that  $\nabla F(x)=A(x-x^*)$, so  $\|\nabla F(x)\|^2\geq 2\lambda_{\min}(A) F(x)^2$. Also note that 
\[
F(x)=\frac12 (x-x^*)^TA(x-x^*)+\frac12\geq \frac12\lambda_{\min}(A)\|x-x^*\|^2\geq  \frac14\lambda_{\min}(A)\|x\|^2- \frac14\lambda_{\min}(A)\|x^*\|^2. 
\]
\end{proof}

\begin{proof}[Proof of Proposition~\ref{prop:mf}]
The gradient and Hessian are given by 
	\[
	\nabla F(X)=2(XX^T-M)X \quad\Rightarrow \quad \|\nabla F(X)\|\leq 4\|\Gamma\|^3. 
	\]
	The Hessian can be defined by its product with two specified matrices $Z$:
	\[
	\langle Z,\nabla^2 F(X)Z \rangle =\|XZ^T+ZX^T\|^2_F+2\langle XX^T, ZZ^T\rangle- 2\langle M, ZZ^T\rangle
	\]
	We check $C_2=\max_{\|Z\|_F=1} \langle Z,\nabla^2 F(X_n)Z \rangle$,
	but for any $Z$ with Frobenius norm $1$, we find that 
	\[
	\|X_nZ^T\|_F^2 \leq \|X_n \|^2_F \|Z\|_F^2 \leq \Gamma,\quad \langle X_nX_n^T, ZZ^T\rangle\leq \|X_n\|^2_F \|Z\|^2\leq \Gamma, \langle M,ZZ^T\rangle\geq 0, 
	\]
	So $C_2\leq 24\Gamma$. Next we check $D_4=\tr_+(\nabla^2 f)$. Since $M$ is positive definite, so $\langle M, ZZ^T\rangle\geq 0$, so 
	\[
	\langle Z,\nabla^2 F(X)Z \rangle \leq \|XZ^T+ZX^T\|^2_F+2\langle XX^T, ZZ^T\rangle
	\] 
	Then for each eigenvector $Z$ of $\nabla^2 f$, its eigenvalue $\lambda$ has its positive part bounded by 
	\[
	[\lambda]_+ =[\langle Z,\nabla^2 F(X)Z \rangle]_+ \leq \|XZ^T+ZX^T\|^2_F+2\langle XX^T, ZZ^T\rangle
	\]
	Therefore
	\[
	D_4\leq  \sum_{i,j} \|Xe_{i,j}^T+e_{i,j}X^T\|^2_F+2\langle XX^T, e_{i,j}e_{i,j}^T\rangle\leq (4m+2r) \Gamma.
	\]
	which leads to the same bound. 
	
	Next we verify the bound for $C_3=\max_{v_i,v_j,v_k}|\partial_{v_i}\partial_{v_j} \partial_{v_k}F(X)|$ for three different specified $m\times r$ matrices. Then the derivatives can be computed as 
	\[
	\partial_{v_i} F(X)= 2\langle (XX^T-M) X,v_i\rangle
	\]
	\[
	\partial_{v_j}\partial_{v_i} F(X)= 2\langle (v_jX^T+Xv_j^T) X+(XX^T-M) v_j,v_i\rangle
	\]
	\[
	\partial_{v_k}\partial_{v_j}\partial_{v_i} F(X)= 2\langle (v_jv_k^T+v_kv_j^T) X+(v_jX^T+Xv_j^T)  v_k+(v_kX^T+Xv_k^T) v_j,v_i\rangle
	\]
	Clearly, when $\|v_i\|_F,\|v_j\|_F,\|v_k\|_F=1$, and $\|X\|_F\leq \sqrt{\Gamma}$.
	\[
	|\partial_{v_k}\partial_{v_j}\partial_{v_i} F(X)|\leq 12\sqrt{\Gamma}.
	\]
	The stochastic gradient of this problem is given by 
	\[
	\nabla f(X,\omega)= \langle XX^T-M,\omega \rangle (\omega+\omega^T) X 
	\]
	So the gradient noise is given by $D_n=X_nX_n^T-M$
	\[
	\xi_{n+1}=\nabla f(X_n,\omega_{n+1})-\nabla F(X_{n})=(2D_n-\langle D_n, \omega_{n+1}\rangle(\omega_{n+1}^T+\omega_{n+1})) X_n. 
	\]
	
	Since $\|D_n\|_F\leq 2\Gamma$, by Lemma \ref{lem:sparsematrix},
	\[
	\E_n \|\xi_{n+1}\|^4_F\leq 8 \|D_n\|_F^4\Gamma^2+ 8\E\| \langle D_n, \omega\rangle (\omega+\omega^T)X_n\|^4\leq c m^2r^2 \Gamma^6. 
	\]
	So $B_2$ can be chosen as $\sqrt{c}mr\Gamma^3$. 
	This leads to 
	\[
	B_1=\E_n \langle \xi_{n+1}, \nabla^2 F(X_n) \xi_{n+1}\rangle\leq cmr \Gamma^4. 
	\]
	
	To check Assumption \ref{aspt:coercive}, note that 
\[
\|X\|^2_F\leq \|XX^T-M\|_F+\|M\|_F. 
\]
Moreover
\begin{align*}
\|\nabla F(X)\|^2=4 \text{tr}((XX^T-M)^2 XX^T)&\geq 4 \text{tr}((XX^T-M)^3)-4F(X)\lambda_{\max}(M)\\
&\geq \frac{4}{\sqrt{d}} F(X)^{\frac{3}{2}}-4F(X)\lambda_{\max}(M).
\end{align*}
So if $F(X)\geq 2d\lambda^2_{\max}(M)$, then $\|\nabla F(X)\|^2\geq 4\lambda_{\max}(M)F(X)$. As a consequence, the following is always true
\[
\|\nabla F(X)\|^2\geq 4\lambda_{\max}(M)F(X)-8d\lambda^3_{\max}(M). 
\]
\end{proof}

\begin{lem}
	\label{lem:sparsematrix}
	Let $A$ and $B$ be two rank $r$ $m\times m$ symmetric matrices, for Gaussian noise matrices,  there is a universal constant $c$ such that 

	\[
	\E \|\langle B, \omega\rangle (\omega+\omega^T)A\|^4_F\leq cm^2 r^2 \|A\|^4_F\|B\|_F^4. 
	\]
\end{lem}
\begin{proof}[Proof of Lemma~\ref{lem:sparsematrix}]
	Since the distribution of $\omega$ and matrix inner product are invariant under rotation, so we can assume $A$ is diagonal. Let $\lambda_1,\ldots,\lambda_r$ be the diagonal entries of $A$. Then 
	\[
	\langle A,\omega\rangle =\sum_{i=1}^r  \lambda_i \omega_{i,i}\sim \mathcal{N}(0, \|A\|_F^2),\quad \E \langle A,\omega\rangle^8=105\|A\|_F^8. 
	\]
	
	\[
	\E \|\langle B, \omega\rangle (\omega+\omega^T)A\|^4_F\leq \sqrt{\E \langle B, \omega\rangle^8}\sqrt{ \E \|(\omega+\omega^T)A\|^8_F}\leq 16\sqrt{\E \langle B, \omega\rangle^4}\sqrt{ \E \|\omega A\|^8_F}
	\]
	Next we note that with $A$ being diagonal, then 
	\[
	\|\omega A\|^2_F=\sum_{i,j} \lambda^2_j |\omega_{i,j}|^2\leq \|A\|^2_F \sum_{i,j}  |\omega_{i,j}|^2
	\]
	Since $\sum_{i,j}  |\omega_{i,j}|^2\sim \chi_{mr}^2$, so $\E \|\omega A\|^8_F=mr(mr+2)(mr+4)(mr+6)\leq 2m^4 r^4$ for large $m$. 
	%
\end{proof}

\begin{proof}[Proof of Proposition~\ref{prop:pca}]
	The gradient is given by 
	\[
	\nabla F(X)=(XX^T-M)X,\quad \nabla f(X,\omega)= (XX^T-M-\omega )X
	\]
	So the gradient noise is given by 
	\[
	\xi_{n+1}=\nabla f(X_n,\omega_{n+1})-\nabla F(X_{n})=\omega_{n+1} X_n. 
	\]
	The constants $C_2,D_4, c_5, D_5$ concern only of the population loss function $F(X)$, which is identical to the previous proposition. We just need to verify $B_2$ and $B_1$. 
	
	Let $(\lambda_j,v_j)$ be the (eigenvalues, eigenvectors) of  $M$, then the data sample share the same distribution as
	\[
	x_i\sim \sum_{j=1}^d\sqrt{\lambda_j} z_j v_j,\quad z_j\sim \mathcal{N}(0,1)
	\]
	Note that for any combination of $i,j,k,n$,  $\E z_i^2z_j^2z_k^2z_n^2\leq \E z_i^8=105$.
	\[
	\E \|x_ix_i^T\|^4_F=\E (\sum_{j=1}^m \lambda_j z_j^2)^4\leq 105 \tr(M)^4. 
	\]
	So there is a universal $c$ such that 
	\[
	\E \|\omega_i\|^4_F\leq 8\E \|x_ix_i^T\|^4_F+8\|M\|_F^4\leq c\tr(M)^4. 
	\]
	So we can choose $B_2=\tr(M)^2$. 
	
	Next we check $B_1$ and $C_2$. Recall that
	\[
	\langle Z,\nabla^2 F(X)Z \rangle =\|XZ^T+ZX^T\|^2_F+2\langle XX^T, ZZ^T\rangle- 2\langle M, ZZ^T\rangle
	\] 
	So 
	\begin{align*}
	B_1&=\E_n \langle \xi_{n+1}, \nabla^2 F(X_n) \xi_{n+1}\rangle \\
	&=4\E_n\|X_n \omega_{n+1} X^T_n+ \omega_{n+1}X_n^TX_n\|^2_F+8\langle X_n X_n^T-M, \E_n \omega_{n+1}X_n X_n^T\omega^T_{n+1}\rangle .
	\end{align*}
	Then 
	\[
	\|X_n \omega_{n+1} X_n^T\|^2_F=\langle \omega_{n+1}, X_n^TX_n \omega_{n+1} X_n^T X_n\rangle= \| X_n^TX_n \omega_{n+1} \|_F^2\leq \| X_n^TX_n\|^2_F \|\omega_{n+1}\|^2_F
	\]
	The conditional expectation is bounded by $c\Gamma^2 \tr(M)^2$, and likewise for $\|\omega_{n+1}X_n^TX_n\|^2_F$
	So 
	\[
	B_1\leq  \E_n\left(8  \| \omega_{n+1} X^T_nX_n\|^2_F+8 \|\omega_{n+1}X_n^TX_n\|^2_F+8\|\omega_{n+1}X_n X_n^T\|^2_F+8\langle M, \omega_{n+1}X_n X_n^T\omega^T_{n+1}\rangle \right)\leq c\Gamma^2 \tr(M)^2. 
	\]
	As for the empirical loss function, we notice the gradient and Hessian of it satisfies 
	\[
	\partial_{v_i} f(X,\omega)= 2\langle (XX^T-xx^T) X,v_i\rangle
	\]
	\[
	\partial_{v_j}\partial_{v_i} f(X,\omega)= 2\langle (v_jX^T+Xv_j^T) X+(XX^T-M-\omega) v_j,v_i\rangle
	\]
	\[
	\partial_{v_k}\partial_{v_j}\partial_{v_i} f(X,\omega)= 2\langle (v_jv_k^T+v_kv_j^T) X+(v_jX^T+Xv_j^T)  v_k+(v_kX^T+Xv_k^T) v_j,v_i\rangle
	\]
	Under our condition, it is clear that for any $v_i, v_j, v_k$ of Frobenious norm $1$,
	\[
	|\partial_{v_i} f(X,\omega)|\leq (\Gamma+\|\omega\|_F)\Gamma^\frac12,\quad |\partial_{v_j}\partial_{v_i} f(X,\omega)|\leq (\Gamma+\|\omega\|_F),\quad
	|\partial_{v_k}\partial_{v_j}\partial_{v_i} f(X,\omega)|\leq \Gamma^\frac12. 
	\]
\end{proof}

	\bibliography{refs}
	
\end{document}